\documentclass{article}

    \PassOptionsToPackage{numbers, compress}{natbib}

    \usepackage[final]{neurips_2020}

\usepackage{microtype}
\usepackage{graphicx}
\usepackage{booktabs} 

\usepackage{sidecap}
\usepackage[utf8]{inputenc} 
\usepackage[T1]{fontenc}    
\usepackage{hyperref}       
\usepackage{url}            
\usepackage{booktabs}       
\usepackage{amsfonts}       
\usepackage{nicefrac}       
\usepackage{microtype}      
\usepackage{wrapfig}
\usepackage{amsthm}

\usepackage[normalem]{ulem}
\usepackage{threeparttable}
\usepackage{multirow}
\usepackage{caption}
\usepackage{amsbsy}
\usepackage{subcaption}
\usepackage{enumitem}
\usepackage{amsmath,bm}
\usepackage{amsmath}
\usepackage[ruled, vlined, linesnumbered]{algorithm2e}%
\usepackage{soul}
\usepackage{graphicx}
\usepackage{wrapfig}
\newcommand{\removelatexerror}{\let\@latex@error\@gobble}
\usepackage[dvipsnames]{xcolor}
\usepackage{bbm}
\usepackage{thmtools,thm-restate}

\newtheorem{lemma}{Lemma}

\newcommand{\y}{\textbf{y}}
\newcommand{\p}{\textbf{p}}

\newcommand{\rr}{\textbf{r}}
\newcommand{\W}{\textbf{W}}
\newcommand{\M}{\textbf{M}}
\newcommand{\m}{\textbf{m}}
\newcommand{\uu}{\textbf{u}}
\newcommand{\feng}{\textcolor{black}}

\title{Uncertainty Aware Semi-Supervised Learning on Graph Data}

\author{%
  Xujiang Zhao$^{1}$, Feng Chen$^{1}$, Shu Hu$^{2}$, Jin-Hee Cho$^{3}$ \\
  $^{1}$The University of Texas at Dallas, \texttt{\{xujiang.zhao, feng,chen\}@utdallas.edu} \\
  $^{2}$University at Buffalo, SUNY, \texttt{shuhu@buffalo.edu} \\
  $^{3}$Virginia Tech, \texttt{jicho@vt.edu} \\

}

\begin{document}

\maketitle
\begin{abstract}
Thanks to graph neural networks (GNNs), semi-supervised node classification has shown the state-of-the-art performance in graph data.  However, GNNs have not considered different types of uncertainties associated with class probabilities to minimize risk of increasing misclassification under uncertainty in real life. In this work, we propose a multi-source uncertainty framework using a GNN that reflects various types of predictive uncertainties in both deep learning and belief/evidence theory domains for node classification predictions. By collecting evidence from the given labels of training nodes, the \textit{Graph-based Kernel Dirichlet distribution Estimation} (GKDE) method is designed for accurately predicting node-level Dirichlet distributions and detecting out-of-distribution (OOD) nodes.  We validated the outperformance of our proposed model compared to the state-of-the-art counterparts in terms of misclassification detection and OOD detection based on six real network datasets. We found that dissonance-based detection yielded the best results on misclassification detection while vacuity-based detection was the best for OOD detection. To clarify the reasons behind the results, we provided the theoretical proof that explains the relationships between different types of uncertainties considered in this work.
\end{abstract}

\vspace{-2mm}
\section{Introduction}
\vspace{-2mm}

Inherent uncertainties derived from different root causes have realized as serious hurdles to find effective solutions for real world problems. Critical safety concerns have been brought due to lack of considering diverse causes of uncertainties, resulting in high risk due to misinterpretation of uncertainties (e.g., misdetection or misclassification of an object by an autonomous vehicle).  Graph neural networks (GNNs)~\cite{kipf2017semi, velickovic2018graph} have received tremendous attention in the data science community. Despite their superior performance in semi-supervised node classification and regression, they didn't consider various types of uncertainties in the their decision process.  Predictive uncertainty estimation~\cite{kendall2017uncertainties} using Bayesian NNs (BNNs) has been explored for classification prediction and regression in the computer vision applications, based on aleatoric uncertainty (AU) and epistemic uncertainty (EU). AU refers to data uncertainty from statistical randomness (e.g., inherent noises in observations) while EU indicates model uncertainty due to limited knowledge (e.g., ignorance) in collected data.  In the belief or evidence theory domain, Subjective Logic (SL)~\cite{josang2018uncertainty} considered vacuity (or a lack of evidence or ignorance) as uncertainty in a subjective opinion. Recently other uncertainty types, such as dissonance, consonance, vagueness, and monosonance~\cite{josang2018uncertainty}, have been discussed based on SL to measure them based on their different root causes. 

\vspace{-1mm}
We first considered multidimensional uncertainty types in both deep learning (DL) and belief and evidence theory domains for node-level classification, misclassification detection, and out-of-distribution (OOD) detection tasks.  By leveraging the learning capability of GNNs and considering multidimensional uncertainties, we propose a uncertainty-aware estimation framework by quantifying different uncertainty types associated with the predicted class probabilities.  
In this work, we made the following {\bf key contributions}:
\vspace{-2mm}
\begin{itemize}[leftmargin=*, noitemsep]
\item \textbf{A multi-source uncertainty framework for GNNs}. Our proposed framework first provides the estimation of various types of uncertainty from both DL and evidence/belief theory domains, such as dissonance (derived from conflicting evidence) and vacuity (derived from lack of evidence).  In addition, we designed a Graph-based Kernel Dirichlet distribution Estimation (GKDE) method to reduce errors in quantifying predictive uncertainties.
\item \textbf{Theoretical analysis}:  Our work is the first that provides a theoretical analysis about the relationships between different types of uncertainties considered in this work.  We demonstrate via a theoretical analysis that an OOD node may have a high predictive uncertainty under GKDE.
\item \textbf{Comprehensive experiments for validating the performance of our proposed framework}: Based on the six real graph datasets, we compared the performance of our proposed framework with that of other competitive counterparts. We found that the dissonance-based detection yielded the best results in misclassification detection while vacuity-based detection best performed in OOD detection. 
\end{itemize}
\vspace{-2mm}
Note that we use the term `predictive uncertainty' in order to mean uncertainty estimated to solve prediction problems.

\vspace{-2mm}
\section{Related Work} \label{sec:related-work}
\vspace{-2mm}
DL research has mainly considered {\it aleatoric} uncertainty (AU) and {\it epistemic} uncertainty (EU) using BNNs for computer vision applications.  AU consists of homoscedastic uncertainty (i.e., constant errors for different inputs) and heteroscedastic uncertainty (i.e., different errors for different inputs)~\cite{gal2016uncertainty}.  A Bayesian DL framework was presented to simultaneously estimate both AU and EU in regression (e.g., depth regression) and classification (e.g., semantic segmentation) tasks~\cite{kendall2017uncertainties}.  Later, {\em distributional uncertainty} was defined based on distributional mismatch between testing and training data distributions~\cite{malinin2018predictive}.  {\em Dropout variational inference}~\cite{gal2016dropout} was used for an approximate inference in BNNs using epistemic uncertainty, similar to \textit{DropEdge}~\cite{rong2019dropedge}.  Other algorithms have considered overall uncertainty in node classification~\cite{eswaran2017power, liu2020uncertainty, zhang2019bayesian}. However, no prior work has considered uncertainty decomposition in GNNs. 

In the belief (or evidence) theory domain, uncertainty reasoning has been substantially explored, such as Fuzzy Logic~\cite{de1995intelligent}, Dempster-Shafer Theory (DST)~\cite{sentz2002combination}, or Subjective Logic (SL)~\cite{josang2016subjective}.  Belief theory focuses on reasoning inherent uncertainty in information caused by unreliable, incomplete, deceptive, or conflicting evidence.  SL considered predictive uncertainty in subjective opinions in terms of {\em vacuity} (i.e., a lack of evidence) and {\em vagueness} (i.e., failing in discriminating a belief state)~\cite{josang2016subjective}. Recently, other uncertainty types have been studied, such as {\em dissonance} caused by conflicting evidence\cite{josang2018uncertainty}.
In the deep NNs, \cite{sensoy2018evidential} proposed evidential deep learning (EDL) model, using SL to train a deterministic NN for supervised classification in computer vision based on the sum of squared loss.  However, EDL didn't consider a general method of estimating multidimensional uncertainty or graph structure. 

\vspace{-2mm}
\section{Multidimensional Uncertainty and Subjective Logic}
\vspace{-2mm}

This section provides an overview of SL and discusses multiple types of uncertainties estimated based on SL, called {\em evidential uncertainty}, with the measures of \textit{vacuity} and \textit{dissonance}.  In addition, we give a brief overview of {\em probabilistic uncertainty}, discussing the measures of \textit{aleatoric} uncertainty and \textit{epistemic} uncertainty.

\vspace{-2mm}
\subsection{Subjective Logic}\label{SL}
\vspace{-2mm}

A multinomial opinion of a random variable $y$ is represented by $\omega = (\bm{b}, u, \bm{a})$ where a domain is $\mathbb{Y} \equiv \{1, \cdots, K\}$ and the additivity requirement of $\omega$ is given as $\sum_{k \in \mathbb{Y}} b_k + u = 1$.  To be specific, each parameter indicates,
\begin{itemize}
\item $\bm{b}$: {\em belief mass distribution} over $\mathbb{Y}$ and $\bm{b} = [b_1, \ldots, b_K]^T$;
\item $u$: {\em uncertainty mass} representing {\em vacuity of evidence};
\item $\bm{a}$: {\em base rate distribution} over $\mathbb{Y}$ and $\bm{a} = [a_1, \ldots, a_K]^T$.
\end{itemize}
The projected probability distribution of a multinomial opinion can be calculated as:
\begin{equation} \label{eq:multinomial-projected}
P(y=k) = b_k + a_k u,\;\;\; \forall k \in \mathbb{Y}. 
\end{equation}  

A multinomial opinion $\omega$ defined above can be equivalently represented by a $K$-dimensional Dirichlet probability density function (PDF), where the special case with $K=2$ is the Beta PDF as a binomial opinion. 
Let ${\bm \alpha}$ be a strength vector over the singletons (or classes) in $\mathbb{Y}$ and ${\bf p} = [p_1, \cdots, p_K]^T$ be a probability distribution over $\mathbb{Y}$. The Dirichlet PDF 
with ${\bf p}$ as a random vector $K$-dimensional variables is defined by:
\begin{eqnarray} \label{eq:multinomial-dir}
\mathrm{Dir}(\bm{p}| {\bm \alpha}) = \frac{1}{B({\bm \alpha})} \prod\nolimits_{k\in \mathbb{Y}} p_k ^{(\alpha_k-1)},
\end{eqnarray} 
where $\frac{1}{B({\bm \alpha})} = \frac{\Gamma (\sum_{k \in \mathbb{Y}} \alpha_k)}{\prod_{k \in \mathbb{Y}} (\alpha_k)}$, $\alpha_k \geq 0$, and $p_k \neq 0$, if $\alpha_k < 1$.

The term \textit{evidence} is introduced as a measure of the amount of supporting observations collected from data that a sample should be classified 
into a certain class. Let $e_k$ be the evidence derived for the class $k\in \mathbb{Y}$.  The total strength $\alpha_k$ for the  belief of each class $k \in \mathbb{Y}$ can be calculated as: 
$\alpha_k = e_k + a_k W$, 
where $e_k \geq 0, \forall k \in \mathbb{Y}$, and $W$ refers to a non-informative weight representing the amount of uncertain evidence.  Given the Dirichlet PDF as defined above, the expected probability distribution over $\mathbb{Y}$ can be calculated as:
\begin{equation} \label{eq:multinomial-expected}
\mathbb{E}[p_k] = \frac{\alpha_k}{\sum_{k=1}^K \alpha_k} = \frac{e_k+a_k W}{W+\sum_{k=1}^K e_k}. 
\end{equation}
The observed evidence in a Dirichlet PDF can be mapped to a multinomial opinion as follows:
\begin{equation} \label{eq:multinomial-belief}
b_k = \frac{e_k}{S}, \;
u = \frac{W}{S},  
\end{equation}
where $S = \sum_{k=1}^K \alpha_k$ refers to the Dirichlet strength.
Without loss of generality, we set  $a_k = \frac{1}{K}$ and the non-informative prior weight (i.e., $W = K$), which indicates that  $a_k \cdot  W = 1$ for each $k \in \mathbb{Y}$.

\vspace{-1mm}
\subsection{Evidential Uncertainty}
\vspace{-1mm}

In~\cite{josang2018uncertainty}, we discussed a number of multidimensional uncertainty dimensions of a subjective opinion based on the formalism of SL, such as singularity, vagueness, vacuity, dissonance, consonance, and monosonance.  These uncertainty dimensions can be observed from binomial, multinomial, or hyper opinions depending on their characteristics (e.g., the vagueness uncertainty is only observed in hyper opinions to deal with composite beliefs). In this paper, we discuss two main uncertainty types that can be estimated in a multinomial opinion, which are {\em vacuity} and {\em dissonance}. 

The main cause of vacuity is derived from a lack of evidence or knowledge, which corresponds to the uncertainty mass, $u$, of a multinomial opinion in SL as:
$vac(\omega) \equiv u = K/S,$ as estimated in Eq.~(\ref{eq:multinomial-belief}).
This uncertainty exists because the analyst may have  insufficient information or knowledge to analyze the uncertainty. The {\em dissonance} of a multinomial opinion can be derived from the same amount of conflicting evidence and can be estimated based on the difference between singleton belief masses (e.g., class labels), which leads to `inconclusiveness' in decision making applications. For example, a four-state multinomial opinion is given as $(b_1, b_2, b_3, b_4, u, a) = (0.25, 0.25, 0.25, 0.25, 0.0, a)$ based on Eq.~\eqref{eq:multinomial-belief}, although the vacuity $u$ is zero, a decision can not be made if there are the same amounts of beliefs supporting respective beliefs.  Given a multinomial opinion with non-zero belief masses, the measure of dissonance can be calculated as:
\begin{align}
\label{eq:dis}
    diss(\omega)=\sum_{i=1}^{K}\Big(\frac{b_i \sum_{j\neq i}b_j \text{Bal}(b_j,b_i)}{\sum_{j\neq i}b_j}\Big), 
\end{align}
where the relative mass balance between a pair of belief masses $b_j$ and $b_i$ is defined as $\mbox{Bal}(b_j, b_i) = 1 - |b_j - b_i|/(b_j + b_i)$. We note that  the dissonance is measured only when the belief mass is non-zero. If all belief masses equal to zero with vacuity being 1 (i.e., $u=1$), the dissonance will be set to zero.

\vspace{-1mm}
\subsection{Probabilistic Uncertainty}   
\label{sect:multi-dim uncertainty}
\vspace{-1mm}
For classification, the estimation of the probabilistic uncertainty relies on the design of an appropriate Bayesian DL model with parameters $\bm{\theta}$. Given input $x$ and dataset $\mathcal{G}$, we estimate a class probability by $P(y|x) = \int P(y|x;\bm{\theta}) P(\bm{\theta}|\mathcal{G}) d\bm{\theta}$, and obtain \textbf{\textit{epistemic uncertainty}} estimated by mutual information~\cite{depeweg2018decomposition, malinin2018predictive}:
\begin{eqnarray}
\footnotesize
\vspace{-2mm}
\underbrace{I(y, \bm{\theta}|x, \mathcal{G})}_{\text{\textbf{\textit{Epistemic}}}} =\underbrace{\mathcal{H}\big[ \mathbb{E}_{P(\bm{\theta}|\mathcal{G})}[P(y|x;\bm{\theta})] \big]}_{\text{\textbf{\textit{Entropy}}}} -  \underbrace{\mathbb{E}_{P(\bm{\theta}|\mathcal{G})}\big[\mathcal{H}[P(y|x;\bm{\theta})] \big]}_{\text{\textbf{\textit{Aleatoric}}}}, 
\vspace{-2mm} \label{eq:epistemic}
\end{eqnarray}
where $\mathcal{H}(\cdot)$ is Shannon's entropy of a probability distribution. The first term indicates {\bf \textit{entropy}} that represents the total uncertainty while the second term is {\bf \textit{aleatoric}} that indicates data uncertainty.  By computing the difference between entropy and aleatoric uncertainties, we obtain epistemic uncertainty, which refers to uncertainty from model parameters. 

\section{Relationships Between Multiple Uncertainties}
\begin{wrapfigure}{R}{0.45\textwidth}
  \vspace{-5mm} 
  \centering
  \includegraphics[width=0.42\textwidth]{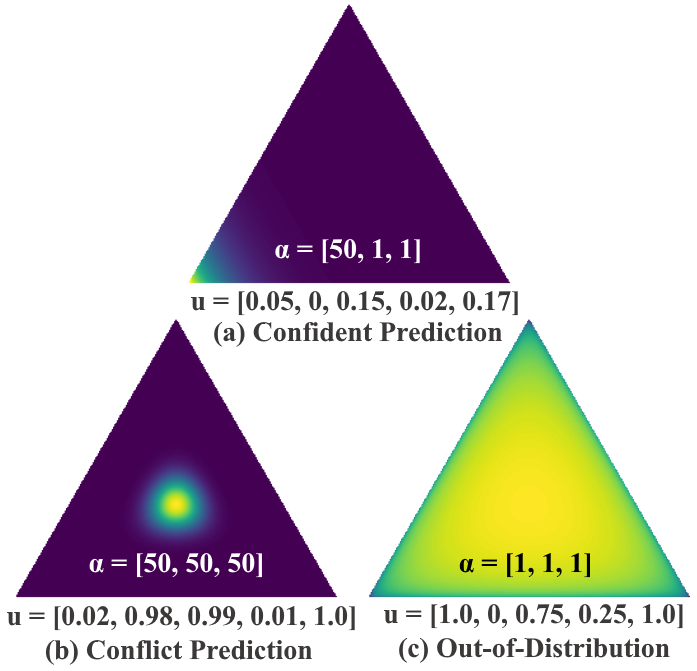}
  \vspace{-2mm}
  \caption{\footnotesize{Multiple uncertainties of different prediction. Let ${\bf u}=[u_v, u_{diss}, u_{alea}, u_{epis}, u_{en}]$.\vspace{-5mm}}}\label{fig:example}
  \vspace{-1mm}
\end{wrapfigure}
We use the shorthand notations $u_{v}$, $u_{diss}$, $u_{alea}$, $u_{epis}$, and $u_{en}$ to represent vacuity, dissonance, aleatoric, epistemic, and entropy, respectively.  

To interpret multiple types of uncertainty, we show three prediction scenarios of 3-class classification in Figure~\ref{fig:example}, in each of which the strength parameters $\alpha = [\alpha_1, \alpha_2, \alpha_3]$ are known.  To make a prediction with high confidence, the subjective multinomial opinion, following a Dirichlet distribution, will yield a sharp distribution on one corner of the simplex (see Figure~\ref{fig:example} (a)). For a prediction with conflicting evidence, called a conflicting prediction (CP), the multinomial opinion should yield a central distribution, representing confidence to predict a flat categorical distribution over class labels (see Figure~\ref{fig:example} (b)).  For an OOD scenario with $\alpha=[1, 1, 1]$, the multinomial opinion would yield a flat distribution over the simplex (Figure~\ref{fig:example} (c)), indicating high uncertainty due to the lack of evidence.  The first technical contribution of this work is as follows.

\begin{restatable}{theorem}{primetheorem}
We consider a simplified scenario, where a multinomial random variable $y$ follows a K-class categorical distribution: $y \sim \text{Cal}(\p)$, the class probabilities $\p$ follow a Dirichlet distribution: $\p\sim \text{Dir}({\bm \alpha})$, and ${\bm \alpha}$ refer to the Dirichlet parameters. Given a total Dirichlet strength $S=\sum_{i=1}^K \alpha_i$, 
for any opinion $\omega$ on a multinomial random variable $y$, we have
\vspace{-2mm}
\begin{enumerate}
\item General relations on all prediction scenarios. 
    
(a) $u_v+ u_{diss} \le 1$; (b) $u_v > u_{epis}$.
    
    \item Special relations on the OOD and the CP.
    \begin{enumerate}
        \item For an OOD sample with a uniform prediction (i.e., $\alpha=[1, \ldots, 1]$), we have 
    \begin{eqnarray}
    1= u_v = u_{en}>  u_{alea} > u_{epis} > u_{diss} = 0 \nonumber
    \end{eqnarray}
        \item For an in-distribution sample with a conflicting prediction (i.e., $\alpha=[\alpha_1, \ldots, \alpha_K]$ with $\alpha_1 = \alpha_2 =\cdots = \alpha_K$, if $S \rightarrow \infty$),  we have 
    \begin{eqnarray}
    u_{en} = 1, \lim_{S\rightarrow \infty} u_{diss} =\lim_{S\rightarrow \infty} u_{alea} =1 , \lim_{S\rightarrow \infty} u_{v} =\lim_{S\rightarrow \infty} u_{epis} =0 \nonumber
    \end{eqnarray}
    \text{with} $u_{en} > u_{alea}> u_{diss}> u_{v}>u_{epis} $.
    \end{enumerate}
\end{enumerate}
\label{theorem1}
\end{restatable}

\vspace{-2mm}
The proof of Theorem~\ref{theorem1} can be found in Appendix A.1.  As demonstrated in Theorem~\ref{theorem1} and Figure~\ref{fig:example}, entropy cannot distinguish OOD (see Figure~\ref{fig:example} (c)) and conflicting predictions (see Figure~\ref{fig:example} (b)) because entropy is high for both cases. Similarly, neither aleatoric uncertainty nor epistemic uncertainty can distinguish OOD from conflicting predictions.  In both cases, aleatoric uncertainty is high while epistemic uncertainty is low.  On the other hand, vacuity and dissonance can clearly distinguish OOD from a conflicting prediction.  For example, OOD objects typically show high vacuity with low dissonance while conflicting predictions exhibit low vacuity with high dissonance.  This observation is confirmed through the empirical validation via our extensive experiments in terms of misclassification and OOD detection tasks.

\vspace{-2mm}
\section{Uncertainty-Aware Semi-Supervised Learning}
\vspace{-2mm}
In this section, we describe our proposed uncertainty framework based on semi-supervised node classification problem. It is designed to predict the subjective opinions about the classification of testing nodes, such that a variety of uncertainty types, such as vacuity, dissonance, aleatoric uncertainty, and epistemic uncertainty, can be quantified based on the estimated subjective opinions and posterior of model parameters. As a subjective opinion can be equivalently represented by a Dirichlet distribution about the class probabilities, we proposed a way to predict the node-level subjective opinions in the form of node-level Dirichlet distributions. The overall description of the framework is shown in Figure~\ref{fig:framework}.

\vspace{-2mm}
\subsection{Problem Definition} \label{subsec:problem-definition}
\vspace{-1mm}
Given an input graph $\mathcal{G} = (\mathbb{V}, \mathbb{E}, {\bf r}, {\bf y}_\mathbb{L})$, where $\mathbb{V} = \{1, \ldots, N \}$ is a ground set of nodes, $\mathbb{E} \subseteq \mathbb{V}\times \mathbb{V}$ is a ground set of edges, $\rr = [\rr_1, \cdots, \rr_N]^T \in \mathbb{R}^{N\times d}$ is a node-level feature matrix, $\rr_i\in \mathbb{R}^d$ is the feature vector of node $i$, $\y_{\mathbb{L}}=\{y_i \mid i \in \mathbb{L}\}$ are the labels of the training nodes $\mathbb{L} \subset \mathbb{V}$, and $y_i \in \{1, \ldots, K\}$ is the class label of node $i$. {\bf We aim to predict}: (1) the \textbf{class probabilities} of the testing nodes: $\p_{\mathbb{V} \setminus \mathbb{L}} = \{\p_i \in [0, 1]^K \mid i \in \mathbb{V} \setminus \mathbb{L}\}$; and (2) the \textbf{associated multidimensional uncertainty estimates} introduced by different root causes: $\mathbf{u}_{\mathbb{V} \setminus \mathbb{L}} = \{\mathbf{u}_i \in [0, 1]^m \mid i \in \mathbb{V} \setminus \mathbb{L}\}$, where $p_{i, k}$ is the probability that the class label $y_i = k$ and $m$ is the total number of
uncertainty types. 

\begin{figure*}[t!]
  \centering
  \vspace{-1mm}
  \includegraphics[width=0.95\linewidth]{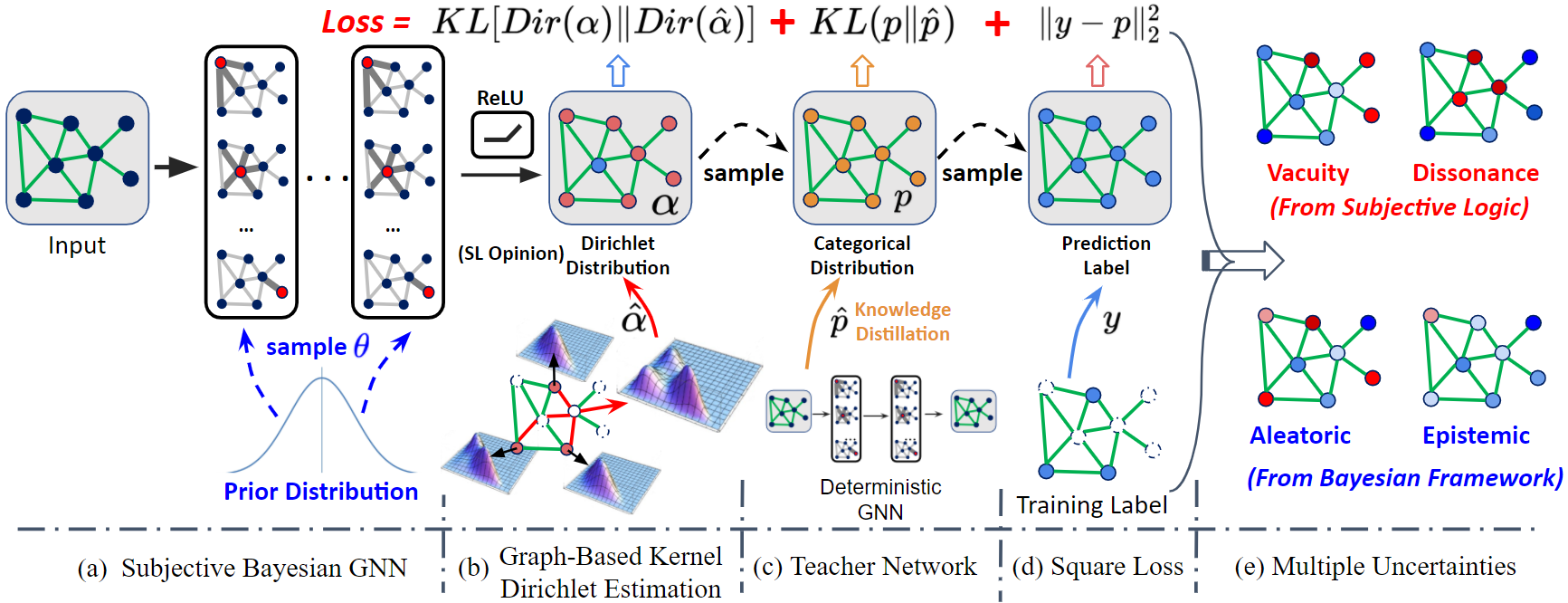}
  \vspace{-2mm}
  \scriptsize{
 \caption{\footnotesize \textbf{Uncertainty Framework Overview.} Subjective Bayesian GNN (a) designed for estimating the different types of uncertainties. The loss function includes square error (d) to reduce bias, GKDE (b) to reduce errors in uncertainty estimation and teacher network (c) to refine class probability.
}
  \label{fig:framework}
\vspace{-4mm}
 }
\end{figure*}  

\vspace{-2mm}
\subsection{Proposed Uncertainty Framework} \label{subsec:bay-dl}
\vspace{-2mm}
\textbf{Learning evidential uncertainty.}
As discussed in Section~\ref{SL}, evidential uncertainty can be derived from multinomial opinions or equivalently Dirichlet distributions to model a probability distribution for the class probabilities. Therefore, we design a Subjective GNN (S-GNN) $f$ to form their multinomial opinions for the node-level Dirichlet distribution $\text{Dir}(\p_i | {\bm \alpha}_i)$ of a given node $i$. Then, the conditional probability $P(\p |A, \rr; \bm{\theta})$ can be obtained by:
\begin{eqnarray}\small 
P(\p |A, \rr; \bm{\theta} )=\prod\nolimits_{i=1}^N \text{Dir}(\p_i|\bm{\alpha}_i), \ \bm{\alpha}_i=f_i(A,\rr;\bm{\theta}), \label{GCN_1}
\end{eqnarray}
where $f_i$ is the output of S-GNN for node $i$, $\bm{\theta}$ is the model parameters, and $A$ is an adjacency matrix. The Dirichlet probability function $\text{Dir}(\p_i | \bm{\alpha}_i)$ is defined by Eq.~\eqref{eq:multinomial-dir}.

Note that S-GNN is similar to classical GNN, except that we use an activation layer (e.g., \textit{ReLU}) instead of the \textit{softmax} layer (only outputs class probabilities).  This ensures that S-GNN would output non-negative values, which are taken as the parameters for the predicted Dirichlet distribution.

\textbf{Learning probabilistic uncertainty.}
Since probabilistic uncertainty relies on a Bayesian framework, we proposed a Subjective Bayesian GNN (S-BGNN) that adapts S-GNN to a Bayesian framework, with the model parameters $\bm{\theta}$ following a prior distribution. The joint class probability of $\y$ can be estimated by: 
\begin{eqnarray}
\vspace{-3mm}
\small 
P(\y |A, \rr; \mathcal{G}) &=& \int \int P(\y | \p) P(\p |A, \rr; \bm{\theta} ) P(\bm{\theta} | \mathcal{G}) d \p d\bm{\theta} \nonumber \\
&\approx& \frac{1}{M}\sum_{m=1}^M  \sum_{i=1}^N \int P(\y_i | \p_i) P(\p_i | A, \rr;\bm{\theta}^{(m)} ) d \p_i, \quad \bm{\theta}^{(m)} \sim q( \bm{\theta}) 
\label{Baye_model}
\vspace{-2mm}
\end{eqnarray}
where $P(\bm{\theta} | \mathcal{G})$ is the posterior, estimated via dropout inference, that provides an approximate solution of posterior $q(\bm{\theta})$ and taking samples from the posterior distribution of models~\cite{gal2016dropout}.  Thanks to the benefit of dropout inference, training a DL model directly by minimizing the cross entropy (or square error) loss function can effectively minimize the KL-divergence between the approximated distribution and the full posterior (i.e., KL[$q(\bm{\theta})\|P(\theta|\mathcal{G})$]) in variational inference~\cite{gal2016dropout, kendall2015bayesian}. For interested readers, please refer to more detail in Appendix B.8.

Therefore, training S-GNN with stochastic gradient descent enables learning of an approximated distribution of weights, which can provide good explainability of data and prevent overfitting.  We use a {\em loss function} to compute its Bayes risk with respect to the sum of squares loss $\|\y-\p\|^2_2$ by:
\begin{eqnarray}\small 
\mathcal{L}(\bm{\theta}) =  \sum\nolimits_{i\in \mathbb{L}} \int \|\y_i-\p_i\|^2_2 \cdot P(\p_i |A, \rr; \bm{\theta}) d \p_i 
=  \sum\nolimits_{i\in \mathbb{L}} \sum\nolimits_{k=1}^K \big(y_{ik}-\mathbb{E}[p_{ik}]\big)^2 + \text{Var}(p_{ik}),
\label{loss}
\end{eqnarray}
where $\y_i$ is an one-hot vector encoding the ground-truth class with $y_{ij} = 1$ and $y_{ik} \neq $ for all $k \neq j$ and $j$ is a class label. Eq.~\eqref{loss} aims to minimize the prediction error and variance, leading to maximizing the classification accuracy of each training node by removing excessive misleading evidence.
\vspace{-1mm}
\subsection{Graph-based Kernel Dirichlet distribution Estimation (GKDE)}
\begin{wrapfigure}{R}{0.4\textwidth}
  \centering
  \includegraphics[width=0.4\textwidth]{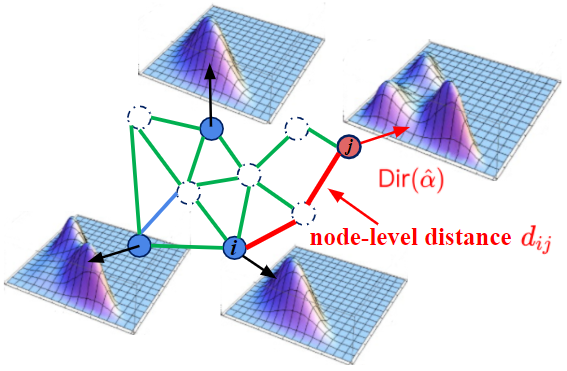}
 \caption{\small{Illustration of GKDE. Estimate prior Dirichlet distribution $\text{Dir}(\hat{\alpha})$ for node $j$ (red) based on training nodes (blue) and graph structure information.}}
  \label{fig:gkde}
  \vspace{-2mm}
\end{wrapfigure} 
\vspace{-1mm}
The loss function in Eq.~\eqref{loss} is designed to measure the sum of squared loss based on class labels of training nodes.  However, it does not directly measure the quality of the predicted node-level Dirichlet distributions.  To address this limitation, we proposed \textit{Graph-based Kernel Dirichlet distribution Estimation} (GKDE) to better estimate node-level Dirichlet distributions by using graph structure information.  The key idea of the GKDE is to estimate prior Dirichlet distribution parameters for each node based on the class labels of training nodes (see Figure~\ref{fig:gkde}). Then, we use the estimated prior Dirichlet distribution in the training process to learn the following patterns: (i) nodes with a high vacuity will be shown far from training nodes; and (ii) nodes with a high dissonance will be shown near the boundaries of classes.

\vspace{-1mm}
Based on SL, let each training node represent one evidence for its class label. Denote the contribution of evidence estimation for node $j$ from training node $i$ by $\mathbf{h}(y_i,d_{ij}) =[h_1, \ldots, h_k, \ldots, h_K] \in[0, 1]^K$, where $h_k(y_i,d_{ij})$ is obtained by: 
\vspace{-1mm}
\begin{eqnarray}
h_k(y_i,d_{ij}) = \begin{cases}0 & y_i \neq k \\ g(d_{ij}) & y_i = k,  \end{cases}
\vspace{-2mm}
\end{eqnarray}
 $g(d_{ij}) = \frac{1}{\sigma \sqrt{2\pi}}\exp({-\frac{{d^2_{ij}}}{2\sigma^2}})$ is the Gaussian kernel function used to estimate the distribution effect between nodes $i$ and $j$, and $d_{ij}$ means the \textbf{node-level distance} (\textbf{a shortest path between nodes $i$ and $j$}), and $\sigma$ is the bandwidth parameter. The prior evidence is estimated based GKDE: $\hat{\bm{e}}_j = \sum_{i\in \mathbb{L}} \mathbf{h}(y_i,d_{ij})$, where $\mathbb{L}$ is a set of training nodes and the prior Dirichlet distribution $\hat{\bm{\alpha}}_j = \hat{\bm{e}}_j +\bf 1$.  
During the training process, we minimize the KL-divergence between model predictions of Dirichlet distribution and prior distribution: $\min \text{KL}[\text{Dir}(\bm{\alpha}) \| \text{Dir}(\hat{\bm{\alpha}})]$.
This process can prioritize the extent of data relevance based on the estimated evidential uncertainty, which is proven effective based on the proposition below.

\begin{restatable}{proposition}{primeproposition}
Given $L$ training nodes, for any testing nodes $i$ and $j$, let ${\bm d}_i = [d_{i1}, \ldots, d_{iL}]$ be the vector of graph distances from nodes $i$ to training nodes and ${\bm d}_j = [d_{j1}, \ldots, d_{jL}]$ be the graph distances from nodes $j$ to training nodes, where $d_{il}$ is the node-level distance between nodes $i$ and $l$. If for all $l\in \{1, \ldots, L\}$, $d_{il} \ge d_{jl}$, then we have
\begin{eqnarray}
\hat{u}_{v_i} \ge \hat{u}_{v_j}, \nonumber
\end{eqnarray}
where $ \hat{u}_{v_i}$ and $\hat{u}_{v_j}$ refer to vacuity uncertainties of nodes $i$ and $j$ estimated based on GKDE.
\label{theorem: vacuity}
\vspace{-1mm}
\end{restatable}
The proof for this proposition can be found in Appendix A.2. The above proposition shows that if a testing node is too far from training nodes, the vacuity will increase, implying that an OOD node is expected to have a high vacuity.

In addition, we designed a simple iterative knowledge distillation method~\cite{hinton2015distilling} (i.e., Teacher Network) to refine the node-level classification probabilities. The key idea is to train our proposed model (Student) to imitate the outputs of a pre-train a vanilla GNN (Teacher) by adding a regularization term of KL-divergence. This leads to solving the following optimization problem: 
\begin{eqnarray}\small 
\vspace{-1mm}
\min\nolimits_{\bm{\theta}}  \mathcal{L}(\bm{\theta}) + \lambda_1 \text{KL}[\text{Dir}({\bm \alpha}) \| \text{Dir}(\hat{\bm \alpha})] + \lambda_2  \text{KL}[P(\y \mid A,\rr;\mathcal{G}) \parallel P(\y|\hat{\p})], 
\label{joint loss}
\vspace{-1mm}
\end{eqnarray}
where $\hat{\p}$ is the vanilla GNN's (Teacher) output and $\lambda_1$ and $\lambda_2$ are trade-off parameters. 

\section{Experiments} \label{sec:exp-results-analysis}
\vspace{-1mm}
In this section, we conduct experiments on the tasks of misclassification and OOD detections to answer the following questions for semi-supervised node classification:

\vspace{-1mm}
\noindent {\bf Q1. Misclassification Detection:} What type of uncertainty is the most promising indicator of high confidence in node classification predictions? 

\vspace{-1mm}
\noindent {\bf Q2. OOD Detection:}  What type of uncertainty is a key indicator of accurate detection of OOD nodes? 

\vspace{-1mm}
\noindent {\bf Q3. GKDE with Uncertainty Estimates:}  How can GKDE help enhance prediction tasks with what types of uncertainty estimates?

\vspace{-1mm}
Through extensive experiments, we found the following answers for the above questions:

\vspace{-1mm}
\noindent {\bf A1.} Dissonance (i.e., uncertainty due to conflicting evidence) is more effective than other uncertainty estimates in misclassification detection. 
\vspace{-1mm}

\noindent {\bf A2.}  Vacuity (i.e., uncertainty due to lack of confidence) is more effective than other uncertainty estimates in OOD detection.

\vspace{-1mm}
\noindent {\bf A3.}  GKDE can indeed help improve the estimation quality of node-level Dirichlet distributions, resulting in a higher OOD detection.

\vspace{-2mm}
\subsection{Experiment Setup} 
\vspace{-2mm}
\textbf{Datasets}: We used six datasets, including three citation network datasets~\cite{sen2008collective} (i.e., Cora, Citeseer, Pubmed) and three new datasets~\cite{shchur2018pitfalls} (i.e., Coauthor Physics, Amazon Computer, and Amazon Photo). We summarized the description and experimental setup of the used datasets in Appendix B.2\footnote{The source code and datasets are accessible at \href{https://github.com/zxj32/uncertainty-GNN}{\color{magenta}{https://github.com/zxj32/uncertainty-GNN}}}.  

\vspace{-1mm}
\textbf{Comparing Schemes}: 
We conducted the extensive comparative performance analysis based on our proposed models and several state-of-the-art competitive counterparts. We implemented all models based on the most popular GNN model, GCN~\cite{kipf2017semi}.  We compared our model (S-BGCN-T-K) against: (1) Softmax-based GCN~\cite{kipf2017semi} with uncertainty measured based on entropy; and (2) Drop-GCN that adapts the Monte-Carlo Dropout~\cite{gal2016dropout, ryu2019uncertainty} into the GCN model to learn probabilistic uncertainty; (3) EDL-GCN that adapts the EDL model~\cite{sensoy2018evidential} with GCN to estimate evidential uncertainty; (4) DPN-GCN that adapts the DPN~\cite{malinin2018predictive} method with GCN to estimate probabilistic uncertainty. We evaluated the performance of all models considered using the area under the ROC (AUROC) curve and area under the Precision-Recall (AUPR) curve in both experiments~\cite{hendrycks2016baseline}.

\vspace{-1mm}
\subsection{Results}
\vspace{-1mm}
\noindent {\bf Misclassification Detection.} The misclassification detection experiment involves detecting whether a given prediction is incorrect using an uncertainty estimate.  Table~\ref{AUPR:uncertainty} shows that S-BGCN-T-K outperforms all baseline models under the AUROC and AUPR for misclassification detection. The outperformance of dissonance-based detection is fairly impressive. This confirms that low dissonance (a small amount of conflicting evidence) is the key to maximize the accuracy of node classification prediction.  We observe the following performance order: ${\tt Dissonance} > {\tt Entropy} \approx {\tt Aleatoric} > {\tt Vacuity} \approx {\tt Epistemic}$, which is aligned with our conjecture: higher dissonance with conflicting prediction leads to higher misclassification detection. We also conducted experiments on additional three datasets and observed similar trends of the results, as demonstrated in Appendix C.
\vspace{-1mm}
\begin{table*}[th!]
\scriptsize
\caption{AUROC and AUPR for the Misclassification Detection.}
\vspace{-2mm}
\centering
\begin{tabular}{c||c|ccccc|ccccc|c}
\hline
\multirow{2}{*}{Data} & \multirow{2}{*}{Model} & \multicolumn{5}{c|}{AUROC} & \multicolumn{5}{c|}{AUPR} & \multirow{2}{*}{Acc} \\  
                  &                   & Va.\tnote{*}& Dis. & Al. & Ep. & En. & Va. & Dis. &  Al. & Ep. &En. & \\ \hline
\multirow{5}{*}{Cora} & S-BGCN-T-K & 70.6 & \textbf{82.4} & 75.3 & 68.8 & 77.7&  90.3  & \textbf{95.4} & 92.4 & 87.8 &93.4 & \textbf{82.0} \\ 
                  &  EDL-GCN &  70.2  &  81.5  & -   &  - &  76.9  &90.0 &  94.6  &  -  & - &  93.6 & 81.5 \\    
                  &    DPN-GCN &  -  &  -  & 78.3   &  75.5 &  77.3  & -  &  -  &  92.4  & 92.0 &92.4 & 80.8 \\  
                  & Drop-GCN &  -  &  -  & 73.9   &  66.7  & 76.9  & -  &  -  &  92.7  & 90.0 &93.6 & 81.3 \\ 
                  &GCN & -  &  -  &  -  &  -   & 79.6&  -  &  -  &  -  &  -  & 94.1 & 81.5 \\ 
                  \hline
\multirow{5}{*}{Citeseer} &  S-BGCN-T-K & 65.4& \textbf{74.0}& 67.2& 60.7&  70.0& 79.8 & \textbf{85.6} & 82.2 & 75.2 &83.5& \textbf{71.0}  \\ 
                   &  EDL-GCN &  64.9  &  73.6  & -   &  - &  69.6  &79.2 &  84.6 &  -  & - &  82.9 & 70.2 \\    
                  &    DPN-GCN &  -  &  -  & 66.0   &  64.9 &  65.5  & -  &  -  & 78.7  & 77.6 &78.1 & 68.1 \\  
                  &                  Drop-GCN &   - &  -  & 66.4 & 60.8 &   69.8 &   -  & -& 82.3 & 77.8  &83.7 & 70.9  \\ 
                  &                  GCN &  -  &  -  &  -  &  -  & 71.4 &  -  &  -  &  -  & -   &83.2 & 70.3 \\ \hline
\multirow{5}{*}{Pubmed} &  S-BGCN-T-K & 64.1 & \textbf{73.3} & 69.3 & 64.2 &  70.7& 85.6 & \textbf{90.8} & 88.8& 86.1   &89.2 & \textbf{79.3} \\
                  &  EDL-GCN &  62.6  &  69.0  & -   &  - &  67.2  &84.6 &  88.9 &  -  & - &  81.7 & 79.0 \\    
                  &    DPN-GCN &  -  &  -  & 72.7   &  69.2 &  72.5  & -  &  -  &  87.8  & 86.8 &87.7 & 77.1 \\  
                  &                  Drop-GCN &  -  &  -  &  67.3 & 66.1 &   67.2&  -  &  -& 88.6 & 85.6   &89.0 & 79.0  \\ 
                  &                  GCN &  -  &  -  &   - & -  & 68.5&  -  &   - & -   &  - &89.2 & 79.0 \\ \hline
\end{tabular}
\begin{tablenotes}\scriptsize
\centering
\item[*] Va.: Vacuity, Dis.: Dissonance, Al.: Aleatoric, Ep.: Epistemic, En.: Entropy 
\end{tablenotes}
\vspace{2mm}
\label{AUPR:uncertainty}
\vspace{-5mm}
\end{table*}

\begin{table*}[th!]
\scriptsize
\caption{AUROC and AUPR for the OOD Detection.}
\vspace{-2mm}
\centering
\begin{tabular}{c||c|ccccc|ccccc}
\hline
\multirow{2}{*}{Data} & \multirow{2}{*}{Model} & \multicolumn{5}{c|}{AUROC} & \multicolumn{5}{c}{AUPR} \\  
                  &                   & Va.\tnote{*}& Dis. & Al. & Ep.  &En. & Va. & Dis. &  Al. & Ep. & En. \\ \hline
\multirow{4}{*}{Cora} & S-BGCN-T-K & \textbf{87.6} & 75.5 & 85.5 & 70.8  & 84.8&  \textbf{78.4} & 49.0 & 75.3 & 44.5 &  73.1  \\ 
                  &    EDL-GCN &            84.5 & 81.0 & & -& 83.3&    74.2 & 53.2 & - & -&71.4  \\ 
                  &       DPN-GCN &  -  &  -  & 77.3   &  78.9 &  78.3  & -  &  -  &  58.5  & 62.8 &  63.0  \\ 
                  &       Drop-GCN &  -  &  -  & 81.9   &  70.5 &  80.9  & -  &  -  &  69.7  & 44.2 &  67.2  \\  
                  &                  GCN & -  &  -  &  -  &  -  &  80.7&  -  &  -  &  -  &  -  & 66.9 \\ \hline
\multirow{4}{*}{Citeseer} &  S-BGCN-T-K & \textbf{84.8} &55.2&78.4 & 55.1 &  74.0& \textbf{86.8} & 54.1 & 80.8 & 55.8 & 74.0  \\ 
                  &  EDL-GCN &                 78.4 &59.4&- & - &  69.1&         79.8 & 57.3 & - & - & 69.0  \\ 
                  & DPN-GCN &   - &  -  & 68.3 & 72.2 &  69.5 &   -  & -& 68.5 & 72.1   & 70.3  \\ 
                  &                  Drop-GCN &   - &  -  & 72.3 & 61.4 &  70.6 &   -  & -& 73.5 & 60.8   & 70.0  \\ 
                  &                  GCN &  -  &  -  &  -  &  -  &  70.8 &  -  &  -  &  -  & -   & 70.2  \\ \hline
\multirow{4}{*}{Pubmed} &  S-BGCN-T-K & \textbf{74.6} &67.9& 71.8 & 59.2 & 72.2& \textbf{69.6} & 52.9 & 63.6& 44.0   &56.5  \\
                  &    EDL-GCN & 71.5 &68.2& - & - &  70.5& 65.3 & 53.1 & -& -   & 55.0  \\
                   & DPN-GCN &  -  &  -  &  63.5 & 63.7 &   63.5&  -  &  -& 50.7 & 53.9   & 51.1  \\ 
                  &                  Drop-GCN &  -  &  -  &  68.7 & 60.8 &   66.7&  -  &  -& 59.7 & 46.7   & 54.8  \\ 
                  &                  GCN &  -  &  -  &   - & -   &  68.3&  -  &   -   &  -  & -&55.3 \\ \hline
\multirow{4}{*}{Amazon Photo} &  S-BGCN-T-K & \textbf{93.4} & 76.4& 91.4& 32.2  & 91.4&         \textbf{ 94.8} & 68.0 & 92.3& 42.3   & 92.5  \\
                  &     EDL-GCN & 63.4 & 78.1& - & - & 79.2&          66.2 & 74.8 & -&-   & 81.2  \\
                  & DPN-GCN &  -  &  -  &  83.6 &  83.6 &   83.6&  -  &  -& 82.6 & 82.4 &   82.5 \\ 
                  &                  Drop-GCN &  -  &  -  &  84.5 & 58.7 &   84.3&  -  &  -& 87.0 & 57.7   &86.9 \\ 
                  &                  GCN &  -  &  -  &   - & -   &   84.4&  -  &   -   &  -  & -&87.0 \\ \hline
\multirow{4}{*}{Amazon Computer} &  S-BGCN-T-K & \textbf{82.3} & 76.6& 80.9& 55.4  & 80.9& \textbf{70.5} & 52.8 & 60.9& 35.9 & 60.6  \\
                  &    EDL-GCN & 53.2 & 70.1& - & - &  70.0&    33.2 & 43.9 & -& - & 45.7  \\
                   &                  DPN-GCN &  -  &  -  &  77.6 & 77.7 & 77.7 &  -  &  -& 50.8 & 51.2 & 51.0 \\ 
                  &                  Drop-GCN &  -  &  -  &  74.4 & 70.5 &  74.3&  -  &  -& 50.0 & 46.7   &  49.8 \\ 
                  &                  GCN &  -  &  -  &   - & -   &   74.0&    - & -   &  -  & -&48.7 \\ \hline
\multirow{4}{*}{Coauthor Physics} &  S-BGCN-T-K & \textbf{91.3} & 87.6& 89.7& 61.8 & 89.8& \textbf{72.2} & 56.6 & 68.1& 25.9 & 67.9  \\
                  &     EDL-GCN & 88.2 & 85.8& - & -  & 87.6&    67.1 & 51.2 & -&- & 62.1  \\
                   &                  DPN-GCN &  -  &  -  & 85.5 &85.6 &   85.5&  -  &  -& 59.8 & 60.2 &   59.8 \\ 
                  &                  Drop-GCN &  -  &  -  &  89.2 & 78.4 &   89.3&  -  &  -& 66.6 & 37.1   &66.5 \\ 
                  &                  GCN &  -  &  -  &   -    &  - & 89.1&  -  &  -   &  -  & -&64.0 \\ \hline
\end{tabular}
\begin{tablenotes}\scriptsize
\centering
\item[*] Va.: Vacuity, Dis.: Dissonance, Al.: Aleatoric, Ep.: Epistemic, En.: Entropy 
\end{tablenotes}
\vspace{2mm}
\label{Table: AUROC_AUPR:ood}
\vspace{-7mm}
\end{table*}

\vspace{-0.3mm}
\noindent {\bf OOD Detection.} This experiment involves detecting whether an input example is out-of-distribution (OOD) given an estimate of uncertainty. For semi-supervised node classification, we randomly selected one to four categories as OOD categories and trained the models based on training nodes of the other categories. Due to the space constraint, the experimental setup for the OOD detection is detailed in Appendix B.3. 

\vspace{-1mm}
In Table~\ref{Table: AUROC_AUPR:ood}, across six network datasets, our vacuity-based detection significantly outperformed the other competitive methods, exceeding the performance of the epistemic uncertainty and other type of uncertainties. This demonstrates that vacuity-based model is more effective than other uncertainty estimates-based counterparts in increasing OOD detection. We observed the following performance order: ${\tt Vacuity} > {\tt Entropy} \approx {\tt Aleatoric} > {\tt Epistemic} \approx {\tt Dissonance}$, which is consistent with the theoretical results as shown in Theorem~\ref{theorem1}.

\vspace{-1mm}
\noindent {\bf Ablation Study.}
We conducted additional experiments (see Table~\ref{Ablation}) in order to demonstrate the contributions of the key technical components, including GKDE, Teacher Network, and subjective Bayesian framework.  The key findings obtained from this experiment are: (1) GKDE can enhance the OOD detection (i.e., 30\% increase with vacuity), which is consistent with our theoretical proof about the outperformance of GKDE in uncertainty estimation, i.e., OOD nodes have a higher vacuity than other nodes; and (2) the Teacher Network can further improve the node classification accuracy.
\vspace{-2mm}
\subsection{Why is Epistemic Uncertainty Less Effective than Vacuity?}
\vspace{-2mm}
Although epistemic uncertainty is known to be effective to improve OOD detection~\cite{gal2016dropout, kendall2017uncertainties} in computer vision applications, our results demonstrate it is less effective than our vacuity-based approach.  The first potential reason is that epistemic uncertainty is always smaller than vacuity (From Theorem~\ref{theorem1}), which potentially indicates that epistemic may capture less information related to OOD. Another potential reason is that the previous success of epistemic uncertainty for OOD detection is limited to supervised learning in computer vision applications, but its  effectiveness for OOD detection was not sufficiently validated in semi-supervised learning tasks.  Recall that epistemic uncertainty (i.e., model uncertainty) is calculated based on mutual information (see Eq.~\eqref{eq:epistemic}).  In a semi-supervised setting, the features of unlabeled nodes are also fed to a model for training process to provide the model with a high confidence on its output.  For example, the model output $P(\y|A, \rr;\theta)$ would not change too much even with differently sampled parameters $\bm{\theta}$, i.e., $P(\y|A, \rr;\theta^{(i)})\approx P(\y|A, \rr;\theta^{(j)})$, which result in a low epistemic uncertainty.  We also designed a semi-supervised learning experiment for image classification and observed a consistent pattern with the results demonstrated in Appendix C.6. 

\begin{table*}[th!]
\scriptsize
\caption{Ablation study of our proposed models: (1) {\tt S-GCN}: Subjective GCN with vacuity and dissonance estimation; (2) {\tt S-BGCN}: S-GCN with Bayesian framework; (3) {\tt S-BGCN-T}: S-BGCN with a Teacher Network; (4) {\tt S-BGCN-T-K}: S-BGCN-T with GKDE to improve uncertainty estimation.}
\vspace{1mm}
\centering
\vspace{-3mm}
\begin{tabular}{c||c|ccccc|ccccc|c}
\hline
\multirow{2}{*}{Data} & \multirow{2}{*}{Model} & \multicolumn{5}{c|}{AUROC (Misclassification Detection)} & \multicolumn{5}{c|}{AUPR (Misclassification Detection)} & \multirow{2}{*}{Acc} \\
                  &                   & Va.\tnote{*}& Dis. & Al. & Ep. & En. & Va. & Dis. &  Al. & Ep. &En. & \\ \hline
\multirow{4}{*}{Cora} & S-BGCN-T-K & 70.6 & 82.4 & 75.3 & 68.8 & 77.7&  90.3  & \textbf{95.4} & 92.4 & 87.8 &93.4 & 82.0 \\ 
 & S-BGCN-T &  70.8  &  \textbf{82.5}  &75.3  & 68.9 & 77.8  &90.4  &  \textbf{95.4}  &  92.6  & 88.0 &93.4 & \textbf{82.2} \\
                  & S-BGCN &  69.8  &  81.4  & 73.9   &  66.7  & 76.9  & 89.4  &  94.3  &  92.3  & 88.0 &93.1 & 81.2 \\  
                   & S-GCN &  70.2  &  81.5  & -   & -  & 76.9  & 90.0  &  94.6  &  -  & - &93.6 & 81.5 \\
                  \hline
 & & \multicolumn{5}{c|}{AUROC (OOD Detection)} & \multicolumn{5}{c|}{AUPR (OOD Detection)} &  \\ \hline 
\multirow{4}{*}{Amazon Photo} &  S-BGCN-T-K & \textbf{93.4} & 76.4& 91.4& 32.2  & 91.4&         \textbf{ 94.8} & 68.0 & 92.3& 42.3   & 92.5 &-  \\
                  &     S-BGCN-T & 64.0 & 77.5& 79.9 & 52.6 & 79.8&          67.0 & 75.3 & 82.0& 53.7   & 81.9&-  \\
                  & S-BGCN & 63.0 & 76.6& 79.8 & 52.7 & 79.7&          66.5 & 75.1 & 82.1& 53.9   & 81.7&- \\ 
                  & S-GCN & 64.0 & 77.1& - & - & 79.6&          67.0 & 74.9 & -& -   & 81.6&- \\ \hline
\end{tabular}
\begin{tablenotes}\scriptsize
\centering
\item[*] Va.: Vacuity, Dis.: Dissonance, Al.: Aleatoric, Ep.: Epistemic, En.: Entropy 
\end{tablenotes}
\vspace{2mm}
\label{Ablation}
\vspace{-5mm}
\end{table*}

\vspace{-2mm}
\section{Conclusion} \label{sec:conclusion}
\vspace{-2mm}
In this work, we proposed a multi-source uncertainty framework of GNNs for semi-supervised node classification. Our proposed framework provides an effective way of predicting node classification and out-of-distribution detection considering multiple types of uncertainty. We leveraged various types of uncertainty estimates from both DL and evidence/belief theory domains. Through our extensive experiments, we found that dissonance-based detection yielded the best performance on misclassification detection while vacuity-based detection performed the best for OOD detection, compared to other competitive counterparts.  In particular, it was noticeable that applying GKDE and the Teacher network further enhanced the accuracy in node classification and uncertainty estimates.

\section*{Acknowledgments}
We would like to thank Yuzhe Ou for providing proof suggestions. 
This work is supported by the National Science Foundation
(NSF) under Grant No \#1815696 and \#1750911.

\section*{Broader Impact}
In this paper, we propose a uncertainty-aware semi-supervised learning framework of GNN for predicting multi-dimensional uncertainties for the task of semi-supervised node classification. Our proposed framework can be applied to a wide range of applications, including computer vision, natural language processing, recommendation systems, traffic prediction, generative models and many more~\cite{zhou2018graph}. Our proposed framework can be applied to predict multiple uncertainties of different roots for GNNs in these applications, improving the understanding of individual decisions, as well as the underlying models. While there will be important impacts resulting from the use of GNNs in general, our focus in this work is on investigating the impact of using our method to predict multi-source uncertainties for such systems. The additional benefits of this method include improvement of safety and transparency in decision-critical applications to avoid overconfident prediction, which can easily lead to misclassification. 

We see promising research opportunities that can adopt our uncertainty framework, such as investigating whether this uncertainty framework can further enhance misclassification detection or OOD detection. To mitigate the risk from different types of uncertainties, we encourage future research to understand the impacts of this proposed uncertainty framework to solve other real world problems.

\newpage
\small
\medskip

\bibliographystyle{abbrv}
\bibliography{reference}

\newpage

\appendix

\centerline{{\Large \textbf{Appendix}}}
\section{Proofs}
\subsection{Theorem 1's Proof}

\primetheorem*

\textbf{Interpretation}. {\bf Theorem 1.1 (a)} implies that increases in both uncertainty types may not happen 
at the same time.  A higher vacuity leads to a lower dissonance, and vice versa (a higher dissonance leads to a lower vacuity).  This indicates that a high dissonance only occurs only when a large amount of evidence is available and the vacuity is low.  {\bf Theorem 1.1 (b)} shows relationships between vacuity and epistemic uncertainty in which vacuity is an upper bound of 
epistemic uncertainty.  Although some existing approaches~\cite{josang2016subjective, sensoy2018evidential}  treat epistemic uncertainty the same as vacuity, it is not necessarily true except for an extreme case \feng{where} a sufficiently large amount of evidence available, making vacuity  \feng{close to} zero. {\bf Theorem 1.2 (a) and (b)} explain how entropy differs from vacuity and/or dissonance.  We observe \feng{that} entropy is 1 when either vacuity or dissonance is 0.  This implies that entropy cannot distinguish different types of uncertainty \feng{due to different root causes}.  For example, \feng{a} high entropy is observed when an example is an either OOD or misclassified example.  Similarly, \feng{a} high aleatoric uncertainty \feng{value} and 
\feng{a} low epistemic uncertainty \feng{value} are observed under both cases.  However, vacuity and dissonance can capture \feng{different} causes of uncertainty \feng{due to lack of information and knowledge and to conflicting evidence, respectively.}  For example, an OOD objects typically show \feng{a} high vacuity \feng{value and a} low dissonance \feng{value} while \feng{a conflicting} \feng{prediction} exhibit\feng{s} \feng{a} low vacuity \feng{and a} high dissonance.
\begin{proof}
    1. (a) 
    Let the opinion $\omega = [b_1, \ldots, b_K, u_v]$, where $K$ is the number of classes, $b_i$ is the belief for class $i$, $u_v$ is the uncertainty mass (vacuity), and $\sum_{i=1}^K b_i + u_v =1$. Dissonance has a upper bound with 
\begin{eqnarray}
u_{diss} &=& \sum_{i=1}^K \Big(\frac{b_i\sum_{j=1, j\neq i}^K b_j \text{Bal}(b_i, b_j)}{\sum_{j=1, j\neq i}^K b_j} \Big) \\
&\le & \sum_{i=1}^K \Big(\frac{b_i\sum_{j=1, j\neq i}^K b_j }{\sum_{j=1, j\neq i}^K b_j} \Big), \quad (\text{since } 0\le \text{Bal}(b_i, b_j) \le 1 ) \nonumber \\
&=& \sum_{i=1}^K b_i, \nonumber
\end{eqnarray}
where $\text{Bal}(b_i, b_j)$ is the relative mass balance, then we have
\begin{eqnarray}
u_{v} + u_{diss} \le \sum_{i=1}^K b_i + u_v = 1.
\end{eqnarray}

1. (b) For the multinomial random variable $y$, we have
\begin{eqnarray}
y \sim \text{Cal}(\p), \quad \p\sim \text{Dir}({\bm \alpha}),
\end{eqnarray}
where $\text{Cal(\p)}$ is the categorical distribution and $ \text{Dir}({\bm \alpha})$ is Dirichlet distribution. Then we have
\begin{eqnarray}
\text{Prob}(y|{\bm \alpha}) = \int \text{Prob}(y|\p) \text{Prob}(\p|{\bm \alpha}) d\p,
\end{eqnarray}
\feng{and} the epistemic \feng{uncertainty is} estimated by mutual information,
\begin{eqnarray}
    \mathcal{I}[y, \p|{\bm \alpha}]
    =  \mathcal{H}\Big[\mathbb{E}_{\text{Prob}(\p|{\bm \alpha})}[P(y|\p)]\Big]- \mathbb{E}_{\text{Prob}(\p|{\bm \alpha})}\Big[\mathcal{H}[P(y|\p)] \Big].
\end{eqnarray}
Now we consider another measure of ensemble diversity\feng{:} \textit{Expected Pairwise KL-Divergence} between each model in the ensemble. Here the expected pairwise KL-Divergence between two independent distribution\feng{s, including} $P(y|\p_1)$ \feng{and} $P(y|\p_2)$, where $\p_1$ and $\p_2$ are two independent samples from $\text{Prob}(\p|{\bm \alpha})$, \feng{can be} computed,
\begin{eqnarray}
    \mathcal{K}[y, \p|{\bm \alpha}] &=& \mathbb{E}_{\text{Prob}(\p_1|{\bm \alpha}\text{Prob}(\p_2|{\bm \alpha})}\Big[KL[P(y| \p_1)\| P(y| \p_2)]   \Big]  \\
    &=& -\sum_{i=1}^K \mathbb{E}_{\text{Prob}(\p_1|{\bm \alpha})}[P(y|\p_1)] \mathbb{E}_{\text{Prob}(\p_2|{\bm \alpha})}[\ln P(y|\p_2)]  - \mathbb{E}_{\text{Prob}(\p|{\bm \alpha})}\Big[\mathcal{H} [P(y|\p)] \Big] \nonumber \\
    &\ge& \mathcal{I}[y, \p|{\bm \alpha}], \nonumber
\end{eqnarray}
where $\mathcal{I}[y, \p_1|{\bm \alpha}] = \mathcal{I}[y, \p_2|{\bm \alpha}]$. \feng{W}e consider Dirichlet ensemble, the \textit{Expected Pairwise KL Divergence},
\begin{eqnarray}
    \mathcal{K}[y, \p|{\bm \alpha}] &=& -\sum_{i=1}^K \frac{\alpha_i}{S} \Big(\psi(\alpha_i ) - \psi(S)  \Big)  - \sum_{i=1}^K -\frac{\alpha_i}{S}\Big(\psi(\alpha_i + 1) - \psi(S + 1)  \Big) \nonumber \\
    &=& \frac{K-1}{S},
\end{eqnarray}
where $S = \sum_{i=1}^K \alpha_i$ and $\psi(\cdot)$ is the \textit{digamma Function}, which is the derivative of the natural logarithm of the gamma function.
Now \feng{we obtain} the relation\feng{s} between vacuity and epistemic,
\begin{eqnarray}
    \underbrace{\frac{K}{S}}_{\text{Vacuity}} >\mathcal{K}[y, \p|{\bm \alpha}] = \frac{K-1}{S}
    \ge  \underbrace{\mathcal{I}[y, \p|{\bm \alpha}]}_{\text{Epistemic}}.
\end{eqnarray}

2. (a) For \feng{an} out-of-distribution sample, $\alpha=[1, \ldots, 1]$,
\feng{the} vacuity \feng{can be calculated as}
\begin{eqnarray}
u_v &=& \frac{K}{\sum_{i=1}^K \alpha_i} = \frac{K}{K} = 1, 
\end{eqnarray}
and the belief mass $b_i = (\alpha_i - 1)/\sum_{i=1}^K \alpha_i= 0$, we estimate dissonance,
\begin{eqnarray}
u_{diss} &=& \sum_{i=1}^K \Big(\frac{b_i\sum_{j=1, j\neq i}^K b_j \text{Bal}(b_i, b_j)}{\sum_{j=1, j\neq i}^K b_j} \Big) = 0.
\end{eqnarray}
Given the expected probability $\hat{p} = [1/K, \ldots, 1/K]^\top$, the entropy is calculated based on $\log_K$,
\begin{eqnarray}
u_{en} = \mathcal{H}[\hat{p}] =-\sum_{i=1}^K \hat{p}_i \log_K \hat{p}_i = -\sum_{i=1}^K \frac{1}{K} \log_K \frac{1}{K} = \log_K \frac{1}{K}^{-1} =\log_K K = 1\feng{,}
\end{eqnarray}
where $\mathcal{H}(\cdot)$ is the entropy.  Based on Dirichlet distribution, the aleatoric uncertainty refers to the expected entropy,
\begin{eqnarray}
u_{alea} &=& \mathbb{E}_{p\sim \text{Dir}(\alpha)}[\mathcal{H}[p]] \\
&=& -\sum_{i=1}^K \frac{\Gamma (S)}{\prod_{i=1}^K\Gamma(\alpha_i)} \int_{S_K} p_i\log_K p_i \prod _{i=1}^K p_i^{\alpha_i-1} d {\bm p} \nonumber \\
&=& - \frac{1}{\ln K}\sum_{i=1}^K \frac{\Gamma (S)}{\prod_{i=1}^K\Gamma(\alpha_i)} \int_{S_K} p_i\ln p_i \prod _{i=1}^K p_i^{\alpha_i-1} d {\bm p} \nonumber \\
&=& -\frac{1}{\ln K}\sum_{i=1}^K \frac{\alpha_i}{S} \frac{\Gamma(S+1)}{\Gamma(\alpha_i +1)\prod_{i'=1, \neq i}^K \Gamma(\alpha_{i'})} \int_{S_K} p_i^{\alpha_i}\ln p_i \prod _{i'=1, \neq i}^K p_{i'}^{\alpha_{i'}-1} d {\bm p} \nonumber \\
&=& \frac{1}{\ln K}\sum_{i=1}^K \frac{\alpha_i}{S} \big(\psi(S+1) - \psi(\alpha_i +1) \big) \nonumber \\
&=& \frac{1}{\ln K}\sum_{i=1}^K \frac{1}{K}(\psi(K+1)-\psi(2)) \nonumber \\
&=& \frac{1}{\ln K} (\psi(K+1)-\psi(2)) \nonumber \\
&=&\frac{1}{\ln K} ( \psi(2) +\sum_{k=2}^K \frac{1}{k}-\psi(2)) \nonumber \\
&=&\frac{1}{\ln K} \sum_{k=2}^K \frac{1}{k} <\frac{1}{\ln K} \ln K = 1, \nonumber
\label{Eq_alea}
\end{eqnarray}
where $S= \sum_{i=1}^K \alpha_i$\feng{,} ${\bm p} = [p_1, \ldots, p_K]^\top$, \feng{and} $K\ge 2$ is the number of category. The epistemic uncertainty \feng{can be calculated via} the mutual information,
\begin{eqnarray}
u_{epis} &=& \mathcal{H}[\mathbb{E}_{p\sim \text{Dir}(\alpha)}[p]] - \mathbb{E}_{p\sim \text{Dir}(\alpha)}[\mathcal{H}[p]] \\
&=& \mathcal{H}[\hat{p}] - u_{alea} \nonumber \\ 
&=& 1 - \frac{1}{\ln K} \sum_{k=2}^K \frac{1}{k} < 1. \nonumber
\end{eqnarray}
To compare aleatoric \feng{uncertainty} \feng{with} epistemic uncertainty, we first prove that aleatoric uncertainty (Eq.~\eqref{Eq_alea}) is monotonically increasing and converging to 1 \feng{as $K$ increases}. \feng{B}ased on \textit{Lemma~\ref{lemma1}}\feng{,} we have
\begin{eqnarray}
&&\Big(\ln(K+1)-\ln K \Big)\sum_{k=2}^K \frac{1}{k} < \frac{\ln K}{K+1} \nonumber \\
&&\Rightarrow \ln(K+1)\sum_{k=2}^K \frac{1}{k}  < \ln K \Big(\sum_{k=2}^K \frac{1}{k}  + \frac{1}{K+1} \Big) = \ln K \sum_{k=2}^{K+1} \frac{1}{k} \nonumber \\
&& \Rightarrow \frac{1}{\ln K} \sum_{k=2}^K \frac{1}{k}  < \frac{1}{\ln (K+1)} \sum_{k=2}^{K+1} \frac{1}{k} \label{use_lemma1}\feng{.}
\end{eqnarray}
\feng{B}ased on Eq.~\eqref{use_lemma1} and Eq.~\eqref{Eq_alea}, we prove that aleatoric uncertainty is monotonically increasing with respect to $K$. So the minimum aleatoric \feng{can be shown to} be $\frac{1}{\ln 2} \frac{1}{2}$\feng{,} when $K=2$.

\feng{Similarly,} for epistemic uncertainty, which is monotonically decreasing as $K$ increases based on \textit{Lemma~\ref{lemma1}}, the maximum epistemic \feng{can be shown to} be $1- \frac{1}{\ln 2} \frac{1}{2}$ when $K=2$. Then we have,
\begin{eqnarray}
u_{alea} \ge  \frac{1}{\ln 2} \frac{1}{2} > 1 - \frac{1}{2\ln 2} \ge u_{epis} 
\end{eqnarray}

Therefore, we prove \feng{that} $1= u_v = u_{en}>  u_{alea} > u_{epis} > u_{diss} = 0 $.
    
2. (b) For a conflict\feng{ing} prediction, i.e., $\alpha=[\alpha_1, \ldots, \alpha_K]$\feng{,}  with $\alpha_1 = \alpha_2 =\cdots = \alpha_K =C$, and $S=\sum_{i=1}^K \alpha_i = CK$, the expected  probability $\hat{p}=[1/K, \ldots, 1/K]^\top$, the belief mass $b_i=(\alpha_i-1)/S$, \feng{and the} vacuity \feng{can be calculated} as
\begin{eqnarray}
u_v &=& \frac{K}{S} \xrightarrow[]{S\rightarrow \infty} 0, 
\end{eqnarray}
and the dissonance \feng{can be calculated} as
\begin{eqnarray}
u_{diss} &=& \sum_{i=1}^K \Big(\frac{b_i\sum_{j=1, j\neq i}^K b_j \text{Bal}(b_i, b_j)}{\sum_{j=1, j\neq i}^K b_j} \Big) = \sum_{i=1}^K b_i \\
&=& \sum_{i=1}^K\left(\frac{a_i-1}{\sum_{i=1}^Ka_i}\right) \nonumber \\
&=&\frac{\sum_{i=1}^K a_i-k}{\sum_{i=1}^K a_i} \nonumber\\
&=&1-\frac{K}{S}\xrightarrow[]{S\rightarrow \infty} 1. \nonumber
\end{eqnarray}
Given the expected probability $\hat{p} = [1/K, \ldots, 1/K]^\top$, the entropy \feng{can be calculated based on Dirichlet distribution},
\begin{eqnarray}
u_{en} &=& \mathcal{H}[\hat{p}] = \sum_{i=1}^K \hat{p}_i \log_K \hat{p}_i = 1\feng{,}
\end{eqnarray}
and \feng{the} aleatoric uncertainty is estimated as the expected entropy,
\begin{eqnarray}
u_{alea} &=& \mathbb{E}_{p\sim \text{Dir}(\alpha)}[\mathcal{H}[p]] \\
&=& -\sum_{i=1}^K \frac{\Gamma (S)}{\prod_{i=1}^K\Gamma(\alpha_i)} \int_{S_K} p_i\log_K p_i \prod _{i=1}^K p_i^{\alpha_i-1} d {\bm p} \nonumber \\
&=& - \frac{1}{\ln K}\sum_{i=1}^K \frac{\Gamma (S)}{\prod_{i=1}^K\Gamma(\alpha_i)} \int_{S_K} p_i\ln p_i \prod _{i=1}^K p_i^{\alpha_i-1} d {\bm p} \nonumber \\
&=& -\frac{1}{\ln K}\sum_{i=1}^K \frac{\alpha_i}{S} \frac{\Gamma(S+1)}{\Gamma(\alpha_i +1)\prod_{i'=1, \neq i}^K \Gamma(\alpha_{i'})} \int_{S_K} p_i^{\alpha_i}\ln p_i \prod _{i'=1, \neq i}^K p_{i'}^{\alpha_{i'}-1} d {\bm p} \nonumber \\
&=& \frac{1}{\ln K}\sum_{i=1}^K \frac{\alpha_i}{S} \big(\psi(S+1) - \psi(\alpha_i +1) \big) \nonumber \\
&=& \frac{1}{\ln K}\sum_{i=1}^K \frac{1}{K}(\psi(S+1)-\psi(C + 1)) \nonumber \\
&=& \frac{1}{\ln K} (\psi(S+1)-\psi(C + 1)) \nonumber \\
&=&\frac{1}{\ln K} ( \psi(C + 1) +\sum_{k=C + 1}^S \frac{1}{k}-\psi(C + 1)) \nonumber \\
&=& \frac{1}{\ln K} \sum_{k=C+1}^S \frac{1}{k} \xrightarrow[]{S\rightarrow \infty} 1. \nonumber 
\end{eqnarray}
The epistemic \feng{uncertainty can be calculated via} mutual information,
\begin{eqnarray}
u_{epis} &=& \mathcal{H}[\mathbb{E}_{p\sim \text{Dir}(\alpha)}[p]] - \mathbb{E}_{p\sim \text{Dir}(\alpha)}[\mathcal{H}[p]] \\
&=& \mathcal{H}[\hat{p}] - u_{alea} \nonumber \\ 
&=& 1 - \frac{1}{\ln K} \sum_{k=C+1}^S \frac{1}{k} \xrightarrow[]{S\rightarrow \infty} 0. \nonumber 
\end{eqnarray}
Now we compare aleatoric uncertainty \feng{with} vacuity, 
\begin{eqnarray}
u_{alea} &=& \frac{1}{\ln K} \sum_{k=C+1}^S \frac{1}{k} \\
&=& \frac{1}{\ln K} \sum_{k=C+1}^{CK} \frac{1}{k} \nonumber \\
&=& \frac{\ln(CK+1)-\ln(C+1)}{\ln K} \nonumber \\ 
&=& \frac{\ln(K-\frac{K-1}{C+1})}{\ln K} \nonumber \\ 
&>& \frac{\ln(K-\frac{K-1}{2})}{\ln K} \nonumber \\ 
&=& \frac{\ln(4/K+4/K + 1/2)}{\ln K} \nonumber \\ 
&\ge& \frac{\ln[3(4/K+4/K + 1/2)^{\frac{1}{3}}]}{\ln K} \nonumber \\ 
&=& \frac{\ln 3 + \frac{1}{3} \ln (\frac{K^2}{32})}{\ln K} \nonumber \\ &=& \frac{\ln 3 + \frac{2}{3} \ln K - \frac{1}{3}\ln 32}{\ln K} > \frac{2}{3}. \nonumber
\label{eq_alea2}
\end{eqnarray}
\feng{B}ased on Eq.~\eqref{eq_alea2}, when $C > \frac{3}{2}$, we have
\begin{eqnarray}
u_{alea} > \frac{2}{3} > \frac{1}{C} = u_v
\end{eqnarray}
We have already prove\feng{d} that $u_v>u_{epis}$, when $u_{en}=1$, we have $u_{alea}>u_{diss}$
Therefore, we prove \feng{that} $u_{en}>  u_{alea} > u_{diss} >  u_{v} > u_{epis}$ with $u_{en} = 1, u_{diss}\rightarrow 1, u_{alea}\rightarrow 1, u_{v}\rightarrow 0, u_{epis}\rightarrow 0$
\end{proof}

\begin{lemma}
For all integer $N\ge 2$, we have $\sum_{n=2}^N\frac{1}{n} < \frac{\ln N}{(N+1)\ln (\frac{N+1}{N})}$. \label{lemma1}
\end{lemma}

\begin{proof}
We will prove by induction that, for all integer $N\ge 2$,
\begin{eqnarray}
\sum_{n=2}^N\frac{1}{n} < \frac{\ln N}{(N+1)\ln (\frac{N+1}{N})} \label{statement1}.
\end{eqnarray}
\textit{Base case}: \feng{W}hen $N=2$\feng{,} we have $ \frac{1}{2} < \frac{\ln 2}{3 \ln \frac{3}{2}}$ \feng{and} Eq.~\eqref{statement1} is true for $N=2$.

\textit{Induction step}: \feng{L}et \feng{the} integer $K\ge 2$ is given and suppose Eq.~\eqref{statement1} is true for $N=K$, then
\begin{eqnarray}
\sum_{k=2}^{K+1}\frac{1}{k}= \frac{1}{K + 1} + \sum_{k=2}^{K}\frac{1}{k} 
< \frac{1}{K + 1} + \frac{\ln K}{(K+1)\ln (\frac{K+1}{K})} 
= \frac{\ln(K + 1)}{(K + 1)\ln (\frac{K + 1}{K})}. 
\label{step1}
\end{eqnarray}

Denote that $g(x) = (x+1)\ln (\frac{x+1}{x})$ with $x> 2$\feng{.} \feng{W}e get its derivative, $g'(x) = \ln (1+ \frac{1}{x})- \frac{1}{x}< 0$\feng{,} such that $g(x)$ is monotonically decreasing, which results in $g(K)> g(K + 1)$\feng{.} \feng{B}ased on Eq.~\eqref{step1} we have,
\begin{eqnarray}
\sum_{k=2}^{K+1}\frac{1}{k}< \frac{\ln (K +1)}{g(K)} < \frac{\ln (K + 1)}{g(K + 1)} = \frac{\ln(K + 1)}{(K + 2)\ln (\frac{K + 2}{K+1})}. 
\end{eqnarray}
Thus, Eq.~\eqref{statement1} holds for $N=K+1$, and the proof of the induction step is complete.

\textit{Conclusion}: By the principle of induction, Eq.~\eqref{statement1} is true for all integer $N\ge 2$.
\end{proof}

\subsection{Proposition 1's Proof}

\primeproposition*

\textbf{Interpretation}. From the above proposition, if a testing node is too distant (far away) from training nodes, the vacuity increases, indicating that an OOD node is expected to have \feng{a} high vacuity \feng{value}.

\begin{proof}
Let ${\bm y} = [y_1, \ldots, y_L]$ \feng{be} the label vector for training nodes. \feng{B}ased on GKDE, the evidence contribution for the node $i$ and a training node $l\in \{1, \ldots, L|\}$ is ${\bm h} (y_l, d_{il})=[h_1(y_l, d_{il}), \ldots, h_K(y_l, d_{il})]$, where
\begin{eqnarray}
h_k(y_l, d_{il}) = \begin{cases}0 & y_l \neq k\\ g(d_{il}) = \frac{1}{\sigma \sqrt{2\pi}}\exp({-\frac{{d_{il}}^2}{2\sigma^2}})  & y_l = k \end{cases} \label{po3_hk},
\end{eqnarray}
and the prior evidence \feng{can be estimated} based GKDE: 
\begin{eqnarray}
\hat{{\bm e}}_i = \sum_{m=1}^{L} \sum_{k=1}^K h_k(y_l, d_{il}),
\end{eqnarray}
\feng{where} $\hat{{\bm e}}_i=[e_{i1},...,e_{iK}]$. Since each training node only contributes the same evidence \feng{based on} its label based on Eq.~\eqref{po3_hk}, the total evidence is estimated by all \feng{the} contributing evidence as
\begin{eqnarray}
\sum_{k=1}^K e_{ik} = \sum_{m=1}^{L} \frac{1}{\sigma \sqrt{2\pi}} \exp({-\frac{{d_{il}}^2}{2\sigma^2}}) , \quad \sum_{k=1}^K e_{jk} = \sum_{m=1}^{L} \frac{1}{\sigma \sqrt{2\pi}} \exp({-\frac{{d_{jl}}^2}{2\sigma^2}}),
\end{eqnarray}
where the vacuity values for node $i$ and node $j$ based on GKDE are,
\begin{eqnarray}
\label{eq:i-j-vacuity}
\hat{u}_{v_i} = \frac{K}{\sum_{k=1}^K e_{ik} + K}, \quad \hat{u}_{v_j} = \frac{K}{\sum_{k=1}^K e_{jk} + K}.
\end{eqnarray}

Now, we prove Eq.~\eqref{eq:i-j-vacuity} above.  If $d_{il}\ge d_{jl}$ for $\forall l\in \{1, \ldots, L\}$, we have
\begin{eqnarray}
\sum_{k=1}^K e_{ik} &=& \sum_{m=1}^{L} \frac{1}{\sigma \sqrt{2\pi}} \exp({-\frac{{d_{il}}^2}{2\sigma^2}}) \\
&\le& \sum_{m=1}^{L} \frac{1}{\sigma \sqrt{2\pi}} \exp({-\frac{{d_{jl}}^2}{2\sigma^2}}) \nonumber\\
&=& \sum_{k=1}^K e_{jk}, \nonumber
\end{eqnarray}
such that
\begin{eqnarray}
\hat{u}_{v_i} = \frac{K}{\sum_{k=1}^K e_{ik} + K} \ge \frac{K}{\sum_{k=1}^K e_{jk} + K} = \hat{u}_{v_j}. 
\end{eqnarray}
\end{proof}

\section{Additional Experimental Details}
\subsection{Source code}
The source code and datasets are accessible at \href{https://github.com/zxj32/uncertainty-GNN}{\color{magenta}{https://github.com/zxj32/uncertainty-GNN}}

\subsection{Description of Datasets}

\begin{table*}[ht!]
\small
\caption{Description of datasets and their experimental setup for the node classification prediction.}
\centering
\vspace{-0mm}
  \begin{tabular}{lcccccc}
    \toprule
     & \textbf{Cora} & \textbf{Citeseer}  & \textbf{Pubmed} & \textbf{Co. Physics} & \textbf{Ama.Computer} & \textbf{Ama.Photo} \\
    \midrule
    \textbf{\#Nodes}   & 2,708 & 3,327 &  19,717 &  34, 493 & 13, 381 & 7, 487 \\
    \textbf{\#Edges}   &  5,429 &   4,732 &  44,338 &  282, 455 & 259, 159 & 126, 530   \\
    \textbf{\#Classes}  & 7 & 6 &  3  & 5 & 10 & 8  \\
    \textbf{\#Features}  & 1,433 & 3,703 &  500 &  8,415 & 767 & 745   \\
    \textbf{\#Training nodes}  & 140 & 120 &  60 &  100 & 200 & 160   \\
    \textbf{\#Validation nodes}  & 500 & 500 &  500 &  500  &  500 &  500  \\
    \textbf{\#Test nodes}  & 1,000 & 1,000 &  1,000 &  1000  &  1,000 &  1000 \\
    \bottomrule
  \end{tabular}
   \label{table:datasets-description}

\vspace{-3mm}  
\end{table*}

\noindent {\bf Cora, Citeseer, and Pubmed}~\cite{sen2008collective}: These are citation network datasets, where \feng{each} network is a directed network \feng{in which} a node represents a document and an edge is a citation link, meaning that there exists an edge when $A$ document cites $B$ document, or vice-versa with a direction. Each node's feature vector contains a bag-of-words representation of a document. For simplicity, we don't discriminate the direction of links and treat citation links as undirected edges and construct a binary, symmetric adjacency matrix $\mathbf{A}$. Each node is labeled with the class to which it belongs. 

\noindent {\bf Coauthor Physics, Amazon Computers, and Amazon Photo }~\cite{shchur2018pitfalls}:
Coauthor Physics is the dataset for co-authorship graphs based on the Microsoft Academic Graph from the KDD Cup 2016 Challenge\footnote{KDD Cup 2016 Dataset: Online Available at \url{https://kddcup2016.azurewebsites.net/}}. In the graphs, a node is an author and an edge exists when two authors co-author a paper. A node's features represent the keywords of its papers and the node's class label indicates its most active field of study.  Amazon Computers and Amazon Photo are the segments of an Amazon co-purchase graph~\cite{mcauley2015image}, where a node is a good (i.e., product), an edge exists when two goods are frequently bought together. A node's features are bag-of-words representation of product reviews and the node's class label is the product category.

For all the used datasets, we deal with undirected graphs with 20 training nodes for each category. We chose the same dataset splits as in~\cite{yang2016revisiting} with an additional validation node set of 500 labeled examples for the hyperparameter obtained from the citation datasets, and followed the same dataset splits in~\cite{shchur2018pitfalls} for Coauthor Physics, Amazon Computer, and Amazon Photo datasets, for the fair comparison\footnote{The source code and datasets are accessible at \href{https://github.com/zxj32/uncertainty-GNN}{\color{magenta}{https://github.com/zxj32/uncertainty-GNN}}}. 

\textbf{Metric}: We use\feng{d} the following metrics for our experiments:
\vspace{-1mm}
\begin{itemize}[leftmargin=*]
\item {\em Area Under Receiver Operating Characteristics (AUROC)}: AUROC shows the area under the curve where FPR (false positive rate) is in $x$-axis and TPR (true positive rate) is in $y$-axis. It can be interpreted as the probability that a positive example is assigned a higher detection score than a negative example\cite{fawcett2006introduction}.
A perfect detector corresponds to an AUROC score of 100\%.
\item {\em Area Under Precision-Prediction Curve (AUPR)}: The PR curve is a graph showing the precision=TP/(TP+FP) and recall=TP/(TP+FN) against each other,and AUPR denote\feng{s} the area under the precision-recall curve. The ideal case is when Precision is 1 and Recall is 1. 
\end{itemize}

\subsection{Experimental Setup for Out-of-Distribution (OOD) Detection} 
For OOD detection on semi-supervised node classification, we randomly selected 1-4 categories as OOD categories and trained the models only based on training nodes of the other categories. In this setting, we still train\feng{ed} a model for semi-supervised node classification task, but only part of node categories \feng{were} not \feng{used} for training. Hence, we suppose that our model only outputs partial categories (as we don't know the OOD category), see Table~\ref{tab:ood_data}.  For example, Cora dataset, we train\feng{ed} the model with 80 nodes (20 nodes for each category) with the predictions of 4 categories. Positive ratio is the ratio of out-of-distribution nodes among on all test nodes.
\vspace{-2mm}
\begin{table*}[ht]
\caption{Description of datasets and their experimental setup for the OOD detection.}
\centering
\scriptsize
  \begin{tabular}{lcccccc}
    \toprule
    Dataset & \textbf{Cora} & \textbf{Citeseer}  & \textbf{Pubmed} & \textbf{Co.Physics} & \textbf{Ama.Computer}  & \textbf{Ama.Photo} \\
    \midrule
    \textbf{Number of training categories}   &  4 &  3 &  2 & 3  &5  &  4 \\
    \textbf{Training nodes}   &  80 &  60 &  40 & 60  &100   &80   \\
    \textbf{Test nodes} & 1000& 1000 & 1000 & 1000& 1000  &  1000  \\
    \textbf{Positive ratio}  & 38\% & 55\% & 40.4\% & 45.1\%  &  48.1\% &51.1\% \\
    \bottomrule
  \end{tabular}
    \label{tab:ood_data}
\vspace{-3mm}
\end{table*}

\subsection{Baseline Setting}
In experiment part, we consider\feng{ed} 4 baselines. For GCN, we use{\feng d} the same hyper-parameters as~\cite{kipf2017semi}. For EDL-GCN, we use{\feng d} the same hyper-parameters as GCN, and replace{\feng d} softmax layer to activation layer (Relu) with squares loss~\cite{sensoy2018evidential}. For DPN-GCN, we use{\feng d} the same hyper-parameters as GCN, \feng{a}nd change{\feng d} the softmax layer to activation layer (exponential)\feng{.} \feng{N}ote that \feng{as} we can not generate OOD node, we only use{\feng d} in\feng{-}distribution loss of (see Eq.12 in~\cite{malinin2018predictive}) and ignore{\feng d} the OOD part loss. For Drop-GCN,  we use{\feng d} the same hyper-parameters as GCN, and set Monte Carlo sampling times $M=100$, dropout rate equal to 0.5.

\subsection{Time Complexity Analysis}
S-BGCN has a similar time complexity with GCN while S-BGCN-T has the double complexity of GCN. For a given network where $|\mathbb{V}|$ is the number of nodes, $|\mathbb{E}|$ is the number of edges, $C$ is the number of dimensions of the input feature vector for every node, $F$ is the number of features for the output layer, and $M$ is Monte Carlo sampling times. 

\begin{table*}[ht]
\caption{Big-O time complexity of our method and baseline GCN.}
\centering
\scriptsize
  \begin{tabular}{lccccc}
    \toprule
    Dataset & \textbf{GCN} & \textbf{S-GCN}  & \textbf{S-BGCN} & \textbf{S-BGCN-T} & \textbf{S-BGCN-T-K}   \\
    \midrule
    \textbf{Time Complexity (Train)} & $O(|\mathbb{E}|CF)$ &  $O(|\mathbb{E}|CF)$ & $O(2|\mathbb{E}|CF)$ & $O(2|\mathbb{E}|CF)$  &$O(2|\mathbb{E}|CF)$   \\
    \textbf{Time Complexity (Test)}  & $O(|\mathbb{E}|CF)$ &  $O(|\mathbb{E}|CF)$ & $O(M|\mathbb{E}|CF)$ & $O(M|\mathbb{E}|CF)$ &$O(M|\mathbb{E}|CF)$  \\
    \bottomrule
  \end{tabular}
\label{tab:complexity}
\end{table*}

\subsection{Model Setups for semi-supervised node classification}
Our models \feng{were} initialized using Glorot initialization~\cite{glorot2010understanding} and trained to minimize loss using the Adam SGD optimizer~\cite{kingma2014adam}. For the S-BGCN-T-K model, we use{\feng d} the {\em early stopping strategy}~\cite{shchur2018pitfalls} on Coauthor Physics, Amazon Computer and Amazon Photo datasets while {\em non-early stopping strategy} \feng{was} used in citation datasets (i.e., Cora, Citeseer and Pubmed). We set bandwidth $\sigma=1$ for all datasets in GKDE, and set trade off parameters $\lambda_1=0.001$ for misclassification detection, $\lambda_1=0.1$ for OOD detection and $\lambda_2=\min(1,t/200)$ (where $t$ is the index of a current training epoch) for both task; other hyperparameter configurations are summarized in Table~\ref{table:BGCN-T}. 

For semi-supervised node classification, we use{\feng d} 50 random weight initialization for our models on Citation network datasets. For Coauthor Physics, Amazon Computer and Amazon Photo datasets, we report{\feng ed} the result based on 10 random train/validation/test splits.  In both effect of uncertainty on misclassification and the OOD detection, we report{\feng ed} the AUPR and AUROC results in percent averaged over 50 times of randomly chosen 1000 test nodes in all of test sets (except training or validation set) for all models tested on the citation datasets. For S-BGCN-T-K model in these tasks, we use{\feng d} the same hyperparameter configurations as in Table~\ref{table:BGCN-T}, except S-BGCN-T-K Epistemic using 10,000 epochs to obtain the best result.

\begin{table}[ht]
  \caption{Hyperparameter configurations of S-BGCN-T-K model}
\centering
\scriptsize
  \begin{tabular}{ l c c c c c c } 
    \toprule
 & \textbf{Cora} & \textbf{Citeseer}  & \textbf{Pubmed}   & \textbf{Co.Physics} & \textbf{Ama.Computer}  & \textbf{Ama.Photo}\\
    \midrule
 \textbf{Hidden units}   & 16 & 16 &  16 &  64 & 64 & 64   \\
 \textbf{Learning rate}  &  0.01 & 0.01  & 0.01  &  0.01 & 0.01  & 0.01    \\
 \textbf{Dropout}  & 0.5 & 0.5 &  0.5  & 0.1 & 0.2 & 0.2 \\
 \textbf{$L_2$ reg.strength}  & 0.0005 & 0.0005 & 0.0005  &  0.001 &  0.0001 & 0.0001   \\
 \textbf{Monte-Carlo samples}  & 100 & 100 &  100  &  100 & 100 & 100  \\
 \textbf{Max epoch}  & 200 & 200 &  200 &  100000 & 100000 & 100000  \\
    \bottomrule
\end{tabular}
  \label{table:BGCN-T}
 
\end{table}

\subsection{Pseudo code for Our Algorithms}
\begin{algorithm}[ht]
\small{
\DontPrintSemicolon
\KwIn{$\mathbb{G} = (\mathbb{V}, \mathbb{E}, \rr)$ and $\y_{\mathbb{L}}$}
\KwOut{$\p_{\mathbb{V} \setminus \mathbb{L}}$, $\uu_{\mathbb{V} \setminus \mathbb{L}}$} 
\SetKwBlock{Begin}{function}{end function}
{
  $\ell=0$; \;
  Set hyper-parameters $\eta, \lambda_1, \lambda_2$; \;
  
  Initialize the parameters $\gamma, \beta$; \;
  
  Calculate the prior Dirichlet distribution $\text{Dir}(\hat{\alpha})$; \;
  Pretrain the teacher network to get $\text{Prob}(\y |\hat{\p})$; \;
 \Repeat{convergence }
 {
 
 Forward pass to compute $\bm{\alpha}$, $\text{Prob}(\p_i |A, \rr; \mathcal{G})$ for $i\in\mathbb{V}$;\;
 
 Compute joint probability $\text{Prob}(\y |A, \rr; \mathcal{G})$;\;
 
 Backward pass via the chain-rule the calculate the sub-gradient gradient: $g^{(\ell)} = \nabla_\Theta \mathcal{L}(\Theta)$ \;
 
 Update parameters using step size $\eta$ via 
  $\Theta^{(\ell+1)} = \Theta^{(\ell)} - \eta \cdot g^{(\ell)}$\; 

  $\ell = \ell + 1$;\;
 }
 Calculate $\p_{\mathbb{V} \setminus \mathbb{L}}$, $\uu_{\mathbb{V} \setminus \mathbb{L}}$ 

\Return{$\p_{\mathbb{V} \setminus \mathbb{L}}$, $\uu_{\mathbb{V} \setminus \mathbb{L}}$}
}
\caption{S-BGCN-T-K}\label{algorithm}
}
\end{algorithm}

\subsection{Bayesian Inference with Dropout} \label{subsec:bayes-inf-dropout}
\vspace{-2mm}
The marginalization in Eq.(8) (in main paper) is generally intractable. A dropout technique is used to obtain an approximate solution and use samples from the posterior distribution of models~\cite{gal2016dropout}. Hence, we adopt{\feng ed} a dropout technique in~\cite{gal2015bayesian} for variational inference in Bayesian convolutional neural networks where Bernoulli distributions are assumed over the network's weights. This dropout technique allows us to perform probabilistic inference over our Bayesian DL framework using GNNs.  For Bayesian inference, we identif{\feng ied} a posterior distribution over the network's weights, given the input graph $\mathcal{D}$ and observed labels $\y_{\mathbb{L}}$ by
$\text{Prob}(\bm{\theta} | \mathcal{D})$,
where $\bm{\theta}=\{\W_1, \ldots, \W_L, b_1,...,b_L\}$, $L$ is the total number of layers and $W_i$ refers to the GNN's weight matrices of dimensions $D_i\times D_{i-1}$, and $b_i$ is a bias vector of dimensions $D_i$ for layer $i=1, \cdots,L$.

\vspace{-1mm}
Since the posterior distribution is intractable, we use{\feng d} a \textbf{variational inference} to learn $q(\bm{\theta})$, a distribution over matrices whose
columns are randomly set to zero, approximating the intractable posterior by minimizing the Kullback-Leibler (KL)-divergence between this approximated distribution and the full posterior, which is given by:
\begin{equation}
\text{KL}(q(\bm{\theta})\| \text{Prob}(\bm{\theta} | \mathcal{D})) \label{KL-VI}
\end{equation}
We define $\W_i$ in $q(\bm{\theta})$ by: 
\begin{eqnarray}\small
&\W_i = \M_i \text{diag}([z_{ij}]_{j=1}^{D_i}), 
&z_{ij} \sim \text{Bernoulli}(d_i) \text{ for } i=1, \ldots, L, j=1, \ldots, D_{i-1} 
\vspace{-3mm}
\end{eqnarray}
where $\bm{\gamma}=\{\M_1, \ldots, \M_L, \m_1, \ldots, \m_L\}$
are the variational parameters, $\M_i \in \mathbb{R}^{D_i\times D_{i-1}}$, $\m_i\in \mathbb{R}^{D_i}$, and $\textbf{d}=\{d_1, \ldots, d_L\}$ is the dropout probabilities
with $z_{ij}$ of Bernoulli distributed random variables. The binary variable $z_{ij} = 0$ corresponds to unit $j$ in layer $i-1$ being dropped out as an input to layer $i$. We can obtain the approximate model of the Gaussian process from~\cite{gal2015bayesian}. The dropout probabilities, $d_i$'s, can be optimized or fixed~\cite{kendall2015bayesian}. For simplicity, we fix{\feng ed} $d_i$'s in our experiments, as it is beyond the scope of our study.  In~\cite{gal2015bayesian}, the minimization of the cross entropy (or square error) loss function is proven to minimize the KL-divergence (see Eq.~\eqref{KL-VI}). Therefore, training the GNN model with stochastic gradient descent enables learning of an approximated distribution of weights, which provides good explainability of data and prevents overfitting.

\vspace{-1mm}
For the dropout inference, we performed training \feng{on} a DL model with dropout before every weight layer and dropout at a test time to sample from the approximate posterior (i.e., stochastic forward passes, a.k.a. Monte Carlo dropout; see Eq.~\eqref{MC_dropout}). 
At the test stage, we infer the joint probability by:
\begin{eqnarray}\small 
p(\y | A, \rr; \mathcal{D}) = \int \int \text{Prob}(\y | \p) \text{Prob}(\p |A, \rr; \bm{\theta} ) \text{Prob}(\bm{\theta} | \mathcal{D}) d \p d\bm{\theta}  \nonumber \\
\approx \frac{1}{M}\sum\nolimits_{m=1}^M \int \text{Prob}(\y | \p) \text{Prob}(\p |A, \rr; \bm{\theta}^{(m)} ) d \p, \quad \bm{\theta}^{(m)} \sim q(\bm{\theta}), 
 \label{MC_dropout}
\end{eqnarray}
where $M$ is Monte Carlo sampling times.
We can also infer the Dirichlet parameters $\bm{\alpha}$ as:
\begin{equation}
    \bm{\alpha} \approx \frac{1}{M}\sum\nolimits_{m=1}^M f(A, \rr, \bm{\theta}^{(m)}), \quad \bm{\theta}^{(m)}\sim q(\bm{\theta}). 
\label{get_alpha}
\end{equation}

\section{Additional Experimental Results}
In addition to the uncertainty analysis in Section 5, we also conducted additional experiments. {\bf First}, we conducted an ablation experiment for each component (such as GKDE, Teacher network, Subjective framework and Bayesian framework) we proposed. {\bf Second}, we provide additional uncertainty visualization results in network node classifications for Citeseer dataset.  To clearly understand the effect of different types of uncertainty in classification accuracy and OOD, we used the AUROC and AUPR curves for all types of models considered in this work.

\subsection{Ablation Experiments}
We conducted an additional experiments in order to clearly demonstrate the contributions of the key technical components, including a teacher Network, Graph kernel Dirichlet Estimation (GKDE) and subjective Bayesian framework.  The key findings obtained from this experiment are: (1) The teacher Network can further improve node classification accuracy (i.e., 0.2\%  - 1.5\% increase, as shown in Table~\ref{Table: AUROC_AUPR:ood3}); and (2) GKDE (Graph-Based Kernel Dirichlet Distribution Estimation) using the uncertainty estimates can enhance OOD detection (i.e., 4\% - 30\% increase, as shown in Table~\ref{Table: AUROC_AUPR:ood2}).

\begin{table*}[th!]
\scriptsize
\caption{Ablation experiment on AUROC and AUPR for the Misclassification Detection.}
\label{Table: AUROC_AUPR:ood3}
\vspace{1mm}
\centering
\begin{tabular}{c||c|ccccc|ccccc|c}
\hline
\multirow{2}{*}{Data} & \multirow{2}{*}{Model} & \multicolumn{5}{c|}{AUROC} & \multicolumn{5}{c}{AUPR} \\  
                  &                   & Va.\tnote{*}& Dis. & Al. & Ep. &En. & Va. & Dis. &  Al. & Ep.  &En. & Acc \\ \hline
\multirow{4}{*}{Cora} & S-BGCN-T-K & 70.6 & 82.4 & 75.3 & 68.8 & 77.7&  90.3  & \textbf{95.4} & 92.4 & 87.8 &93.4 & 82.0 \\ 
 & S-BGCN-T &  70.8  &  \textbf{82.5}  &75.3  & 68.9 & 77.8  &90.4  &  \textbf{95.4}  &  92.6  & 88.0 &93.4 & \textbf{82.2} \\
                  & S-BGCN &  69.8  &  81.4  & 73.9   &  66.7  & 76.9  & 89.4  &  94.3  &  92.3  & 88.0 &93.1 & 81.2 \\  
                   & S-GCN &  70.2  &  81.5  & -   & -  & 76.9  & 90.0  &  94.6  &  -  & - &93.6 & 81.5 \\
                  \hline
\multirow{4}{*}{Citeseer} &  S-BGCN-T-K & 65.4& \textbf{74.0}& 67.2& 60.7&  70.0& 79.8 & \textbf{85.6} & 82.2 & 75.2 &83.5 & 71.0 \\ 
                  &  S-BGCN-T & 65.4& 73.9& 67.1& 60.7& 70.1& 79.6 & 85.5 & 82.1 & 75.2 & 83.5 &\textbf{71.3} \\ 
                  &     S-BGCN &   63.9 &  72.1  & 66.1 & 58.9  & 69.2& 78.4 & 83.8   &80.6 &75.6&82.3&70.6  \\ 
                  &                  S-GCN &  64.9  &  71.9  &  -  & - & 69.4 &  79.5  & 84.2  &  -  & -   & 82.5& 71.0  \\ \hline
\multirow{4}{*}{Pubmed} &  S-BGCN-T-K & 63.1 & \textbf{69.9} & 66.5 & 65.3& 68.1& 85.6 & 90.8 & 88.8& 86.1   &89.2 & \textbf{79.3} \\
                  &     S-BGCN-T & 63.2 & \textbf{69.9} & 66.6 & 65.3 & 64.8&  85.6 & \textbf{90.9} & 88.9& 86.0   &89.3 & 79.2 \\
                  &                  S-BGCN &  62.7  &   68.1 & 66.1 &  64.4 & 68.0&  85.4  &  90.5& 88.6 & 85.6   &  89.2&78.8  \\ 
                  &                  S-GCN &  62.9  &   69.5 & -   &  - & 68.0&  85.3  &  90.4 & -   &  -  & 89.2&79.1 \\ \hline
\multirow{4}{*}{Amazon Photo} &  S-BGCN-T-K & 66.0 & 89.3& 83.0& 83.4  & 83.2&   95.4 & \textbf{98.9}& 98.4& 98.1 & 98.4 & 92.0\\
                  &     S-BGCN-T &  66.1 & 89.3& 83.1& 83.5  & 83.3&   95.6 & 99.0& 98.4& 98.2 & 98.4&\textbf{92.3}\\
                  &                  S-BGCN &   68.6  & \textbf{ 93.6} & 90.6 &  83.6 & 90.6&  90.4  &  98.1& 97.3 & 95.8 &  97.3&81.0 \\ 
                  &                  S-GCN &    -  &   - & -   &  - & 86.7&  -  &   - & -   &  -  & -&98.4 \\ \hline
\multirow{4}{*}{Amazon Computer} &  S-BGCN-T-K & 65.0 & 87.8& 83.3& 79.6  & 83.6& 89.4 & 96.3 & 95.0& 94.2 &  95.0 & 84.0\\
                  &     S-BGCN-T & 65.2 & 88.0& 83.4& 79.7  & 83.6& 89.4 & \textbf{96.5} & 95.0& 94.5 & 95.1 &\textbf{84.1} \\
                  &                  S-BGCN &  63.7  &  \textbf{89.1} & 84.3 &  76.1 & 84.4&  84.9  &  95.7& 93.9 & 91.4   &  93.9&76.1 \\ 
                  &                  S-GCN &   -  &   - & -   &  - & 81.5&  -  &   - & -   &  -  & -&95.2 \\ \hline
\multirow{4}{*}{Coauthor Physics} &  S-BGCN-T-K & 80.2 & 91.4& 87.5& 81.7 & 87.6& 98.3 &\textbf{99.4} & 99.0& 98.4 & 98.9 & 93.0\\
                  &     S-BGCN-T & 80.4 & \textbf{91.5}& 87.6& 81.7  & 87.6& 98.3 & \textbf{99.4} & 99.0& 98.6 &  99.0 & \textbf{93.2} \\
                  &                  S-BGCN &  79.6   &  90.5 & 86.3 &  81.2 & 86.4&  98.0  &  99.2& 98.8 & 98.3   &  98.8&92.9 \\ 
                  &                  S-GCN &  89.1   &  89.0 & -   &  - & 89.2&  99.0  & 99.0 & -   &  -  & 99.0&92.9 \\ \hline
\end{tabular}
\begin{tablenotes}\scriptsize
\centering
\item[*] Va.: Vacuity, Dis.: Dissonance, Al.: Aleatoric, Ep.: Epistemic, En.: Entropy 
\end{tablenotes}
\end{table*}

\begin{table*}[th!]
\scriptsize
\caption{Ablation experiment on  AUROC and AUPR for the OOD Detection.}
\vspace{1mm}
\centering
\begin{tabular}{c||c|ccccc|ccccc}
\hline
\multirow{2}{*}{Data} & \multirow{2}{*}{Model} & \multicolumn{5}{c|}{AUROC} & \multicolumn{5}{c}{AUPR} \\  
                  &                   & Va.\tnote{*}& Dis. & Al. & Ep. &En. & Va. & Dis. &  Al. & Ep.  &En. \\ \hline
\multirow{4}{*}{Cora} & S-BGCN-T-K & \textbf{87.6} & 75.5 & 85.5 & 70.8  & 84.8&  \textbf{78.4} & 49.0 & 75.3 & 44.5 &  73.1  \\ 
                  &    S-BGCN-T &            84.5 & 81.2 & 83.5 & 71.8 & 83.5&    74.4 & 53.4 & 75.8 & 46.8 & 71.7  \\ 
                  &                  S-BGCN &  76.3  &  79.3  & 81.5   &  70.5  & 80.6  & 61.3  & 55.8  &  68.9  & 44.2 &  65.3  \\  
                  &                  S-GCN & 75.0  &  78.2  &  -  &  -    & 79.4&  60.1  &  54.5  &  -  &  -  & 65.3 \\ \hline
\multirow{4}{*}{Citeseer} &  S-BGCN-T-K & \textbf{84.8} &55.2&78.4 & 55.1 & 74.0& \textbf{86.8} & 54.1 & 80.8 & 55.8 & 74.0  \\ 
                  &  S-BGCN-T &                 78.6 &59.6&73.9 & 56.1 & 69.3&         79.8 & 57.4 & 76.4 & 57.8 & 69.3  \\ 
                  &                  S-BGCN &   72.7 &  63.9  & 72.4 & 61.4 &  70.5 &   73.0  & 62.7& 74.5 & 60.8   & 71.6  \\ 
                  &                  SGCN &  72.0 &  62.8  &  -  &  -   & 70.0 &  71.4  &  61.3  &  -  & -   & 70.5  \\ \hline
\multirow{4}{*}{Pubmed} &  S-BGCN-T-K & \textbf{74.6} &67.9& 71.8 & 59.2& 72.2& \textbf{69.6} & 52.9 & 63.6& 44.0   &56.5  \\
                  &     S-BGCN-T & 71.8 &68.6& 70.0 & 60.1 & 70.8&          65.7 & 53.9 & 61.8& 46.0   & 55.1  \\
                  &                  S-BGCN &  70.8 &  68.2  &  70.3 & 60.8 & 68.0& 65.4  &  53.2& 62.8 & 46.7   &  55.4  \\ 
                  &                  S-GCN &    71.4  &   68.8 & -   &  - & 69.7&  66.3  &   54.9 & -   &  -  &57.5 \\ \hline
\multirow{4}{*}{Amazon Photo} &  S-BGCN-T-K & \textbf{93.4} & 76.4& 91.4& 32.2  & 91.4&         \textbf{ 94.8} & 68.0 & 92.3& 42.3   & 92.5   \\
                  &     S-BGCN-T & 64.0 & 77.5& 79.9 & 52.6 & 79.8&          67.0 & 75.3 & 82.0& 53.7   & 81.9  \\
                  & S-BGCN & 63.0 & 76.6& 79.8 & 52.7 & 79.7&          66.5 & 75.1 & 82.1& 53.9   & 81.7 \\ 
                  & S-GCN & 64.0 & 77.1& - & - & 79.6&          67.0 & 74.9 & -& -   & 81.6 \\ \hline
\multirow{4}{*}{Amazon Computer} &  S-BGCN-T-K & \textbf{82.3} & 76.6& 80.9& 55.4 & 80.9& \textbf{70.5} & 52.8 & 60.9& 35.9 &  60.6  \\
                  &     S-BGCN-T & 53.7 & 70.5& 70.4 & 69.9  & 70.1&    33.6 & 43.9 & 46.0& 46.8 & 45.9  \\
                  &                  S-BGCN &  56.9  &  75.3  &  74.1 & 73.7 &  74.1&33.7  &  46.2& 48.3 & 45.6   &48.3 \\ 
                  &                  S-GCN &  56.9  &  75.3  & -&  - & 74.2&  33.7  &   46.2 & -   &  - &48.3 \\ \hline
\multirow{4}{*}{Coauthor Physics} &  S-BGCN-T-K & \textbf{91.3} & 87.6& 89.7& 61.8  & 89.8& \textbf{72.2} & 56.6 & 68.1& 25.9 & 67.9  \\
                  &     S-BGCN-T & 88.7 & 86.0& 87.9 & 70.2  & 87.8&    67.4 & 51.9 & 64.6& 29.4 & 62.4  \\
                  &                  S-BGCN &  89.1  &  87.1  &  89.5 & 78.3 &   89.5&  66.1  &  49.2&64.6 & 35.6   & 64.3 \\ 
                  &                  S-GCN &  89.1  &    87.0 & -   &  - & 89.4&  -66.2  &   49.2 & -   &  -  & 64.3 \\ \hline
\end{tabular}
\begin{tablenotes}\scriptsize
\centering
\item[*] Va.: Vacuity, Dis.: Dissonance, Al.: Aleatoric, Ep.: Epistemic, D.En.: Differential Entropy, En.: Entropy 
\end{tablenotes}
\vspace{2mm}
\label{Table: AUROC_AUPR:ood2}
\vspace{-5mm}
\end{table*}

\subsection{Experiment based on GAT model}
We also conducted the semi-supervised node classification based on GAT model~\cite{velickovic2018graph}).Model setup: The S-BGAT-T-K model has two dropout probabilities, which are a dropout on features  and a dropout on attention coefficients, as \feng{shown} in Table~\ref{table:BGAT-T}. We changed the dropout on attention coefficients to 0.4 at the test stage and set trade off parameters $\lambda=\min(1,t/50)$, using the same early stopping strategy ~\cite{velickovic2018graph}. The result are shown in Table~\ref{tab:classification_result gat}. 

\begin{table}[ht]
\centering
\caption{Hyper-parameters of S-BGAT-T-K model} 
\small
  \begin{tabular}{lccc}
    \toprule
     & \textbf{Cora} & \textbf{Citeseer}  & \textbf{Pubmed} \\
    \midrule
    \textbf{Hidden units}    &  64 & 64 & 64 \\
    \textbf{Learning rate}   &  0.01 & 0.01  & 0.01   \\
    \textbf{Dropout}  & 0.6/0.6 & 0.6/0.6 &  0.6/0.6   \\
    \textbf{$L_2$ reg.strength}  & 0.0005 & 0.0005 & 0.001   \\
    \textbf{Monte-Carlo samples}  & 100 & 100 &  100 \\
    \textbf{Max epoch}  &  100000 & 100000 & 100000    \\
    \bottomrule
  \end{tabular}
  \label{table:BGAT-T}
\end{table}

\begin{table}[th!]
\centering

\caption{Semi-supervised node classification accuracy based on GAT}
\small
\label{tab:classification_result gat}
\begin{tabular}{l c c c c } 
    \toprule
 & \textbf{Cora} & \textbf{Citeseer}  & \textbf{Pubmed}  \\
    \midrule
 \textbf{GAT} & 83.0 $\pm$ 0.7 & 72.5 $\pm$ 0.7 & 79.0 $\pm$ 0.3    \\
 \textbf{GAT-Drop} & 82.8 $\pm$ 0.8 & 72.6 $\pm$ 0.7 & 79.0 $\pm$ 0.3    \\
 \textbf{S-GAT} & 83.0 $\pm$ 0.7 & 72.6 $\pm$ 0.6 & 79.0 $\pm$ 0.3    \\
 \textbf{S-BGAT} & 82.9 $\pm$ 0.7 & 72.4 $\pm$ 0.7 & 78.9 $\pm$ 0.3    \\
 \textbf{S-BGAT-T}  & 83.7 $\pm$  0.6& \textbf{73.2 $\pm$ 0.5}& 79.1 $\pm$ 0.2 \\
 \textbf{S-BGAT-T-K}  & \textbf{83.8 $\pm$  0.7}& 73.0 $\pm$ 0.7& \textbf{79.1 $\pm$ 0.2} \\
    \bottomrule
\end{tabular}
\end{table}

\subsection{Misclassification Detection}
For Amazon Photo, Amazon Computer and Coauthor Physics dataset, the misclassification detection results are shown in Tabel~\ref{AUPR:uncertainty2}.
\begin{table*}[th!]
\scriptsize
\caption{AUROC and AUPR for the Misclassification Detection.}
\vspace{-2mm}
\centering
\begin{tabular}{c||c|ccccc|ccccc|c}
\hline
\multirow{2}{*}{Data} & \multirow{2}{*}{Model} & \multicolumn{5}{c|}{AUROC} & \multicolumn{5}{c|}{AUPR} & \multirow{2}{*}{Acc} \\  
                  &                   & Va.\tnote{*}& Dis. & Al. & Ep. & En. & Va. & Dis. &  Al. & Ep. &En. & \\ \hline
\multirow{5}{*}{Amazon Photo} &  S-BGCN-T-K & 66.0 & \textbf{89.3}& 83.0& 83.4 &  83.2&   95.4 & \textbf{98.9}& 98.4& 98.1  & 98.4 & \textbf{92.0} \\
                  &  EDL-GCN &  65.1  &  88.5  & -   &  - &  82.2  &94.6 &  98.1 &  -  & - &  98.0 & 91.2 \\    
                  &    DPN-GCN &  -  &  -  & 81.8   &  80.8 &  81.3  & -  &  -  &  98.1  & 98.0 &98.0 & \textbf{92.0} \\  
                  &                  Drop-GCN &  -  &  -  &  84.5 & 84.4 &   84.6&  -  &  -& 98.2 & 98.1 &98.2 & 91.3 \\ 
                  &                  GCN &  -  &  -  &   - & -   &   86.8&  -  &   - & -   &  -  &98.5 & 91.2 \\ \hline
\multirow{5}{*}{Amazon Computer} &  S-BGCN-T-K & 65.0 & \textbf{87.8}& 83.3& 79.6  & 83.6& 89.4 & \textbf{96.3} & 95.0& 94.2 & 95.0 & 84.0 \\
                  &  EDL-GCN &  64.1  &  86.5  & -   &  - &  82.2  &93.6 &  97.1 &  -  & - &  97.0 & 79.7 \\    
                  &    DPN-GCN &  -  &  -  & 76.8   &  76.0 &  76.3  & -  &  -  &  94.5  & 94.3 &94.4 & \textbf{84.8} \\  
                  &                  Drop-GCN &  -  &  -  &  79.1 & 75.9 &   79.2&  -  &  -& 95.1 & 94.5   &  95.1 & 79.6 \\ 
                  &                  GCN &  -  &  -  &   - & -   &   81.7&  -  &   - & -   &  -  &95.4  & 82.6\\ \hline
\multirow{5}{*}{Coauthor Physics} &  S-BGCN-T-K & 80.2 & \textbf{91.4}& 87.5& 81.7  & 87.6& 98.3 &\textbf{99.4} & 99.0& 98.4 & 98.9 & \textbf{93.0}  \\
                  &  EDL-GCN &  78.8  &  89.5  & -   &  - &  86.2  &96.6 &  97.2 &  -  & - &  97.0 & 92.7 \\    
                  &    DPN-GCN &  -  &  -  & 87.0   &  86.4 &  86.8  & -  &  -  &  99.1  & 99.0 &99.0 & 92.5 \\  
                  &                  Drop-GCN &  -  &  -  &  87.6 & 84.1 &   87.7&  -  &  -& 98.9 & 98.6   &  98.9 & 93.0 \\ 
                  &                  GCN &  -  &  -  &   - & -   &   88.7&  -  &   - & -   &  - &99.0 & 92.8 \\ \hline
\end{tabular}
\begin{tablenotes}\scriptsize
\centering
\item[*] Va.: Vacuity, Dis.: Dissonance, Al.: Aleatoric, Ep.: Epistemic, En.: Entropy 
\end{tablenotes}
\vspace{2mm}
\label{AUPR:uncertainty2}
\vspace{-5mm}
\end{table*}

\subsection{Graph Embedding Representations of Different Uncertainty Types} \label{subsec:uncertain-dist}

\begin{figure*}[th!]
  \centering
  \includegraphics[width=0.7\linewidth, height=0.4\textwidth]{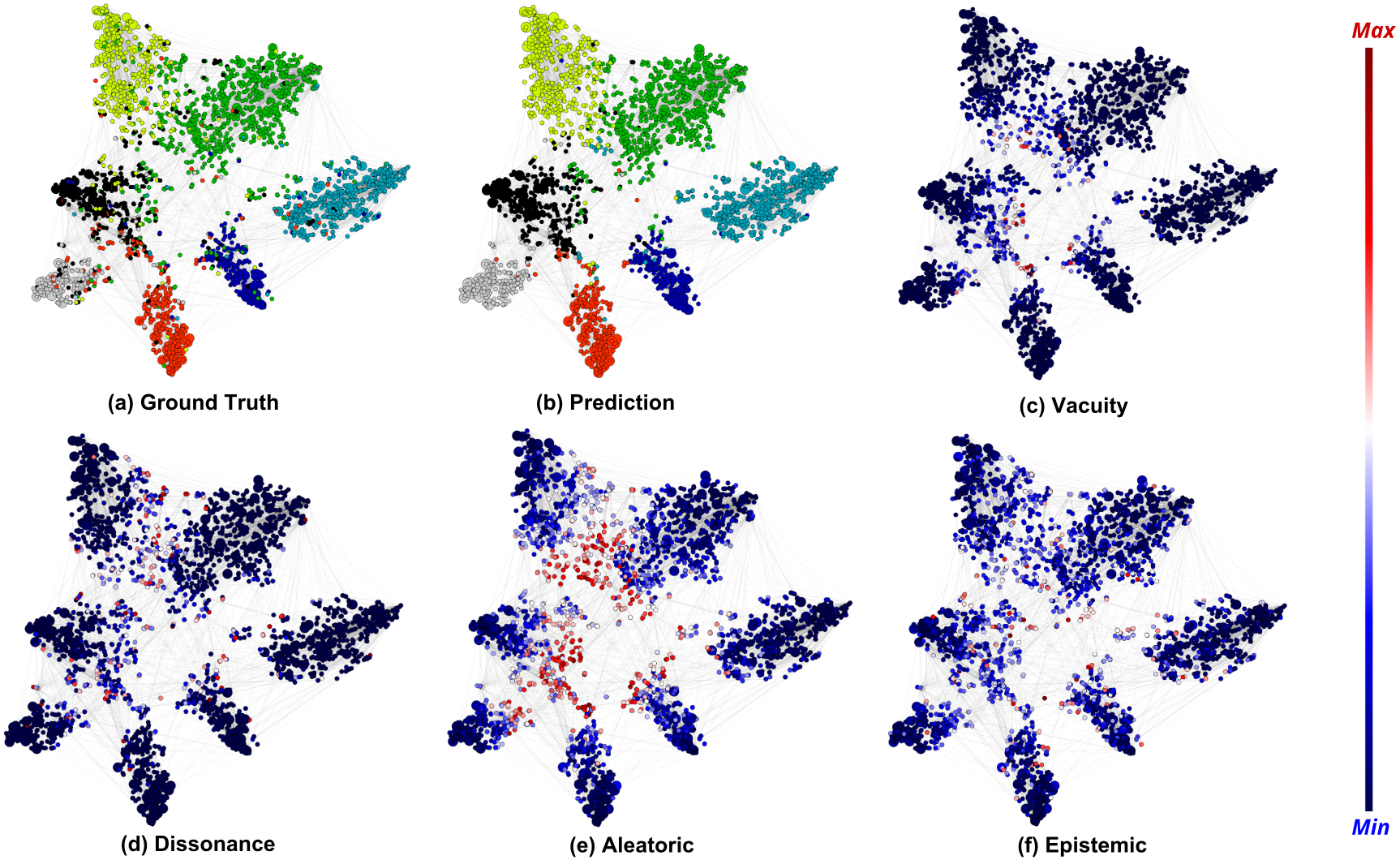}
  \small{
  \caption{Graph embedding representations of the Cora dataset for classes and the extent of uncertainty: (a) shows the representation of seven different classes; (b) shows our model prediction; and (c)-(f) present the extent of uncertainty for respective uncertainty types, including vacuity, dissonance, aleatoric, epistemic.}
  \label{fig:un_vis_cora}
\vspace{-3mm}
 }
\end{figure*}

\begin{figure*}[th!]
  \centering
  \includegraphics[width=0.7\linewidth, height=0.4\textwidth]{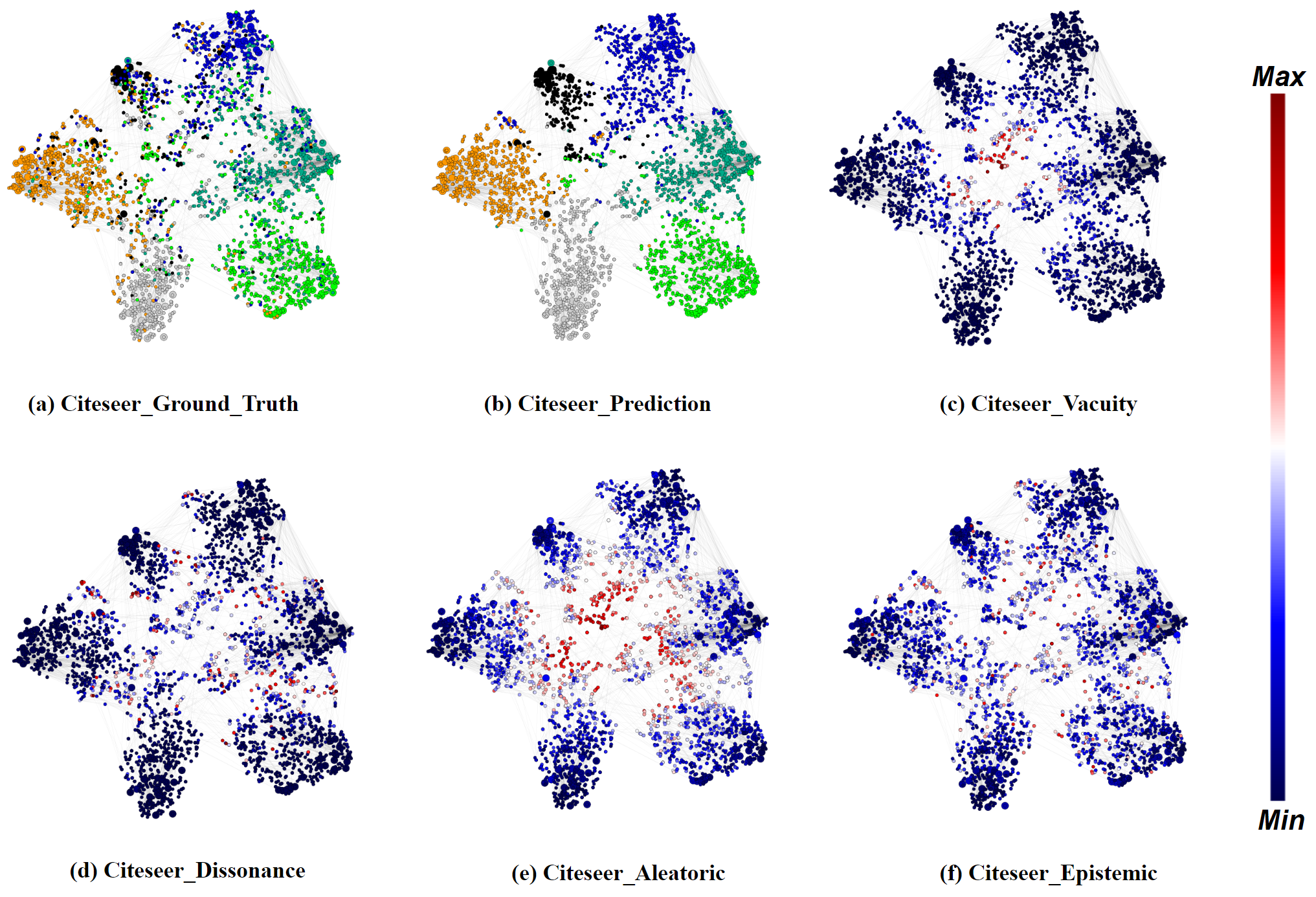}
  \small{
  \caption{Graph embedding representations of the Citeseer dataset for classes and the extent of uncertainty: (a) shows the representation of seven different classes, (b) shows our model prediction and (c)-(f) present the extent of uncertainty for respective uncertainty types, including vacuity, dissonance, and aleatoric uncertainty, respectively.}
  \label{fig:un_vis}
\vspace{-3mm}
 }
\end{figure*}

To better understand different uncertainty types, we used $t$-SNE ($t$-Distributed Stochastic Neighbor Embedding~\cite{maaten2008visualizing}) to represent the computed feature representations of a pre-trained BGCN-T model's first hidden layer on the Cora dataset and the Citeseer dataset. 

\noindent {\bf Seven Classes on Cora Dataset}: In Figure~\ref{fig:un_vis_cora}, (a) shows the representation of seven different classes, (b) shows our model prediction and (c)-(f) present the extent of uncertainty for respective uncertainty types, including vacuity, dissonance, and aleatoric uncertainty, respectively.

\noindent {\bf Six Classes on Citeseer Dataset}: In Figure~\ref{fig:un_vis} (a), a node's color denotes a class on the Citeseer dataset where 6 different classes are shown in different colors. Figure~\ref{fig:un_vis} (b) is our prediction result.

\noindent {\bf Eight Classes on Amazon Photo Dataset}: In Figure~\ref{fig:Ab_GKDE}, a node's color denotes vacuity uncertainty value, and the big node represent training node. These results are based on OOD detection experiment. Compare Figure~\ref{fig:Ab_GKDE} (a) and (b), we found that GKDE can indeed improve the OOD detection.

For Figures~\ref{fig:un_vis} (c)-(f), the extent of uncertainty is presented where a blue color refers to the lowest uncertainty (i.e., minimum uncertainty) while a red color indicates the highest uncertainty (i.e., maximum uncertainty) based on the presented color bar. To examine the trends of the extent of uncertainty depending on either training nodes or test nodes, we draw training nodes as bigger circles than test nodes.  Overall we notice that most training nodes (shown as bigger circles) have low uncertainty (i.e., blue), which is reasonable because the training nodes are the ones that are already observed. Now we discuss the extent of uncertainty under each uncertainty type. 

\vspace{1mm}
\noindent {\bf Vacuity}: In Figure~\ref{fig:Ab_GKDE} (b), most training nodes show low uncertainty, we observe majority of OOD nodes in the button cluster show high uncertainty as appeared in red. 

\vspace{1mm}
\noindent {\bf Dissonance}: In Figure~\ref{fig:un_vis} (d), similar to vacuity, training nodes have low uncertainty. But unlike vacuity, test nodes are much less uncertain. Recall that dissonance represents the degree of conflicting evidence (i.e., discrepancy between each class probability). However, in this dataset, we observe a fairly low level of dissonance and the obvious outperformance of Dissonance in node classification prediction.

\vspace{1mm}
\noindent {\bf Aleatoric uncertainty}: In Figure~\ref{fig:un_vis} (e), a lot of nodes show high uncertainty with larger than 0.5 except a small amount of training nodes with low uncertainty.

\vspace{1mm}
\noindent {\bf Epistemic uncertainty}: In Figure~\ref{fig:un_vis} (f), most nodes show very low epistemic uncertainty \feng{values} because uncertainty derived from model parameters can disappear as they are trained well.

\begin{figure*}[t!]
    \centering
    \begin{subfigure}[b]{0.44\textwidth}
        \centering
        \includegraphics[width=\linewidth]{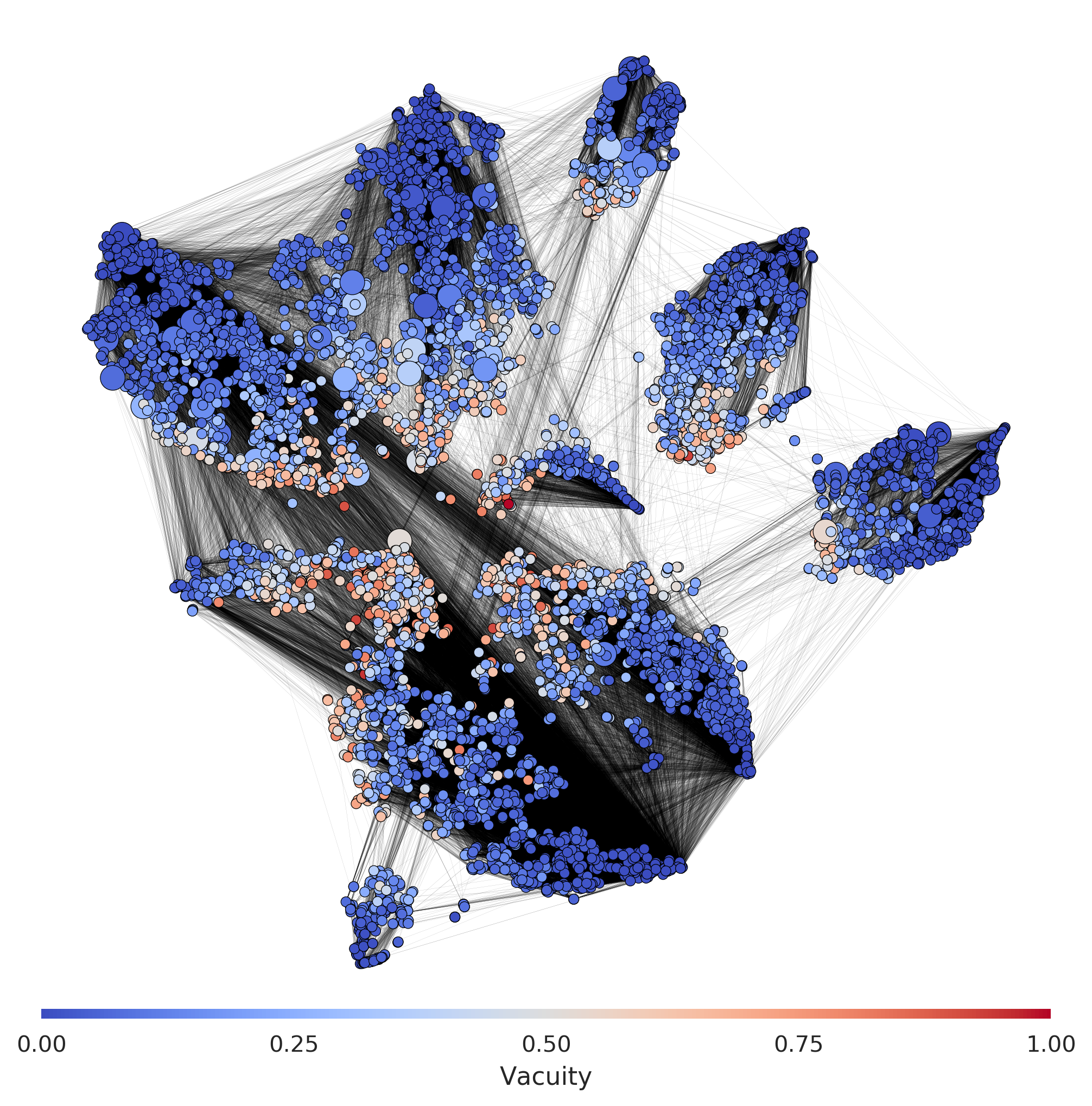}
        \caption{S-BGCN-T}
    \end{subfigure}
    \begin{subfigure}[b]{0.44\textwidth}
        \centering
        \includegraphics[width=\linewidth]{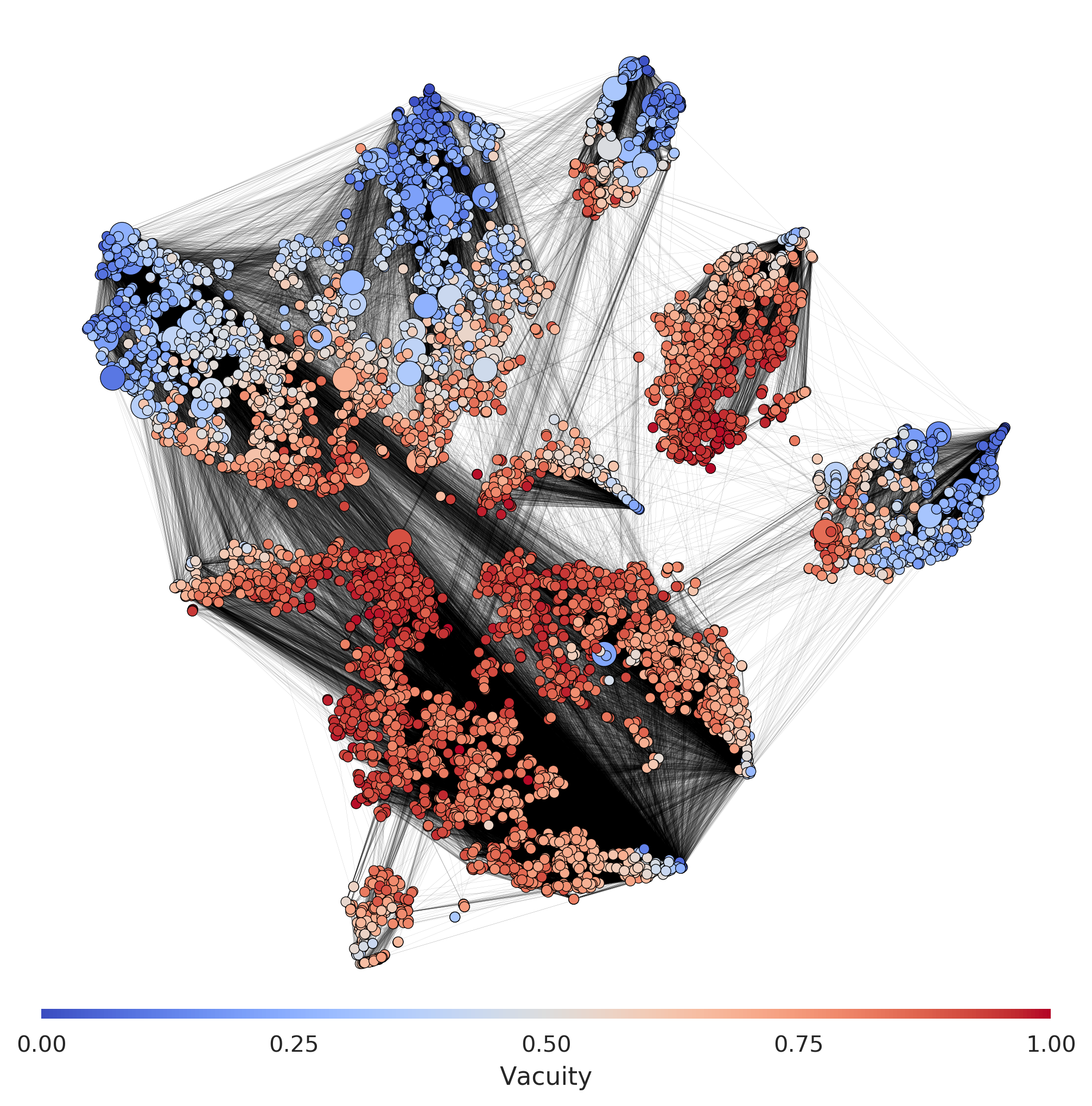}
        \caption{S-BGCN-T-K}
    \end{subfigure}
    \small{
    \caption{Graph embedding representations of the Amazon Photo dataset for the extent of vacuity uncertainty based on OOD detection experiment.}
    \label{fig:Ab_GKDE}
    }
\end{figure*}

\begin{figure*}[t!]
    \centering
    \begin{subfigure}[b]{0.32\textwidth}
        \centering
        \includegraphics[width=\linewidth]{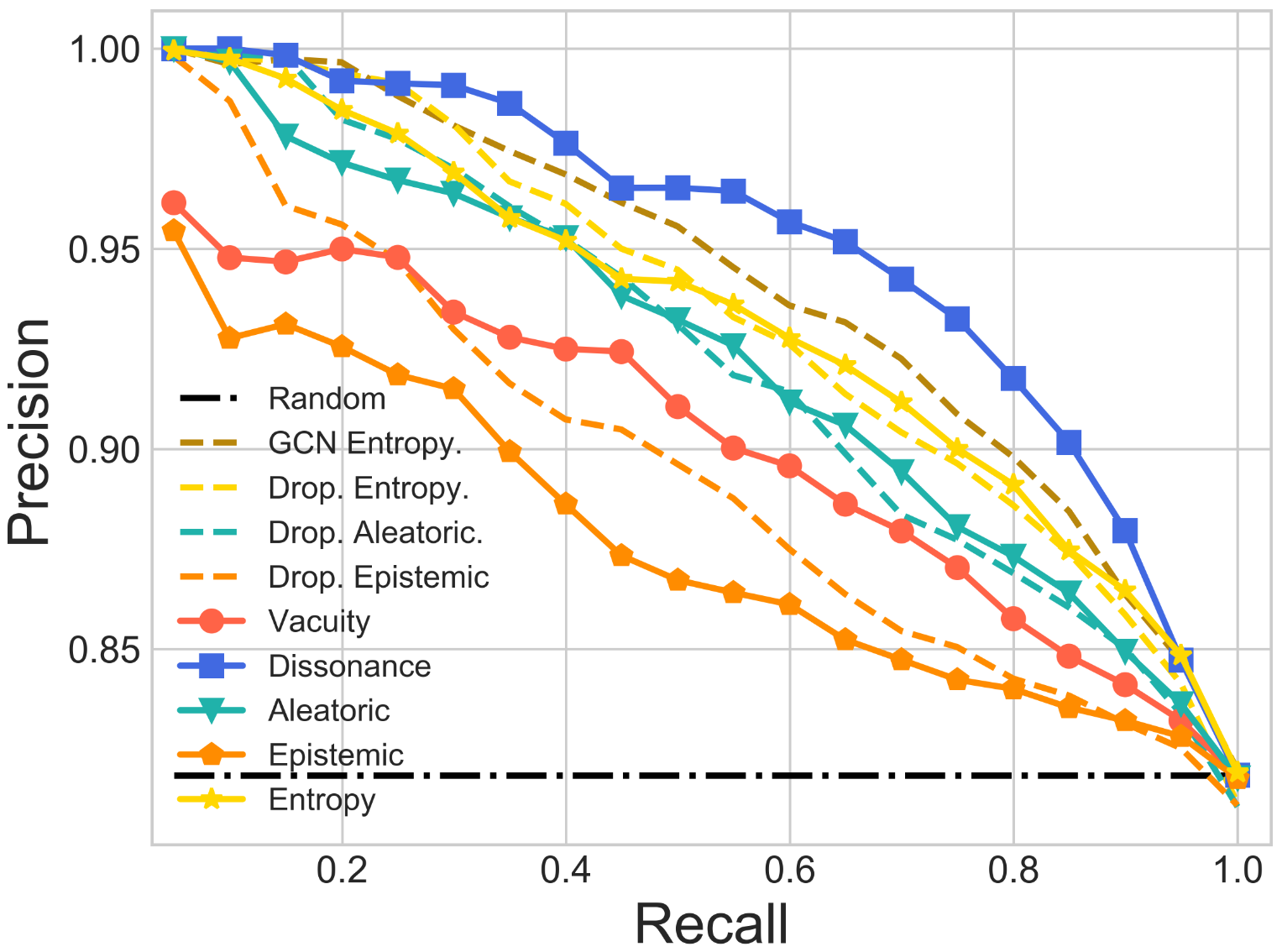}
        \caption{PR curves on Cora}
    \end{subfigure}
    \hfill
    \begin{subfigure}[b]{0.32\textwidth}
        \centering
        \includegraphics[width=\linewidth]{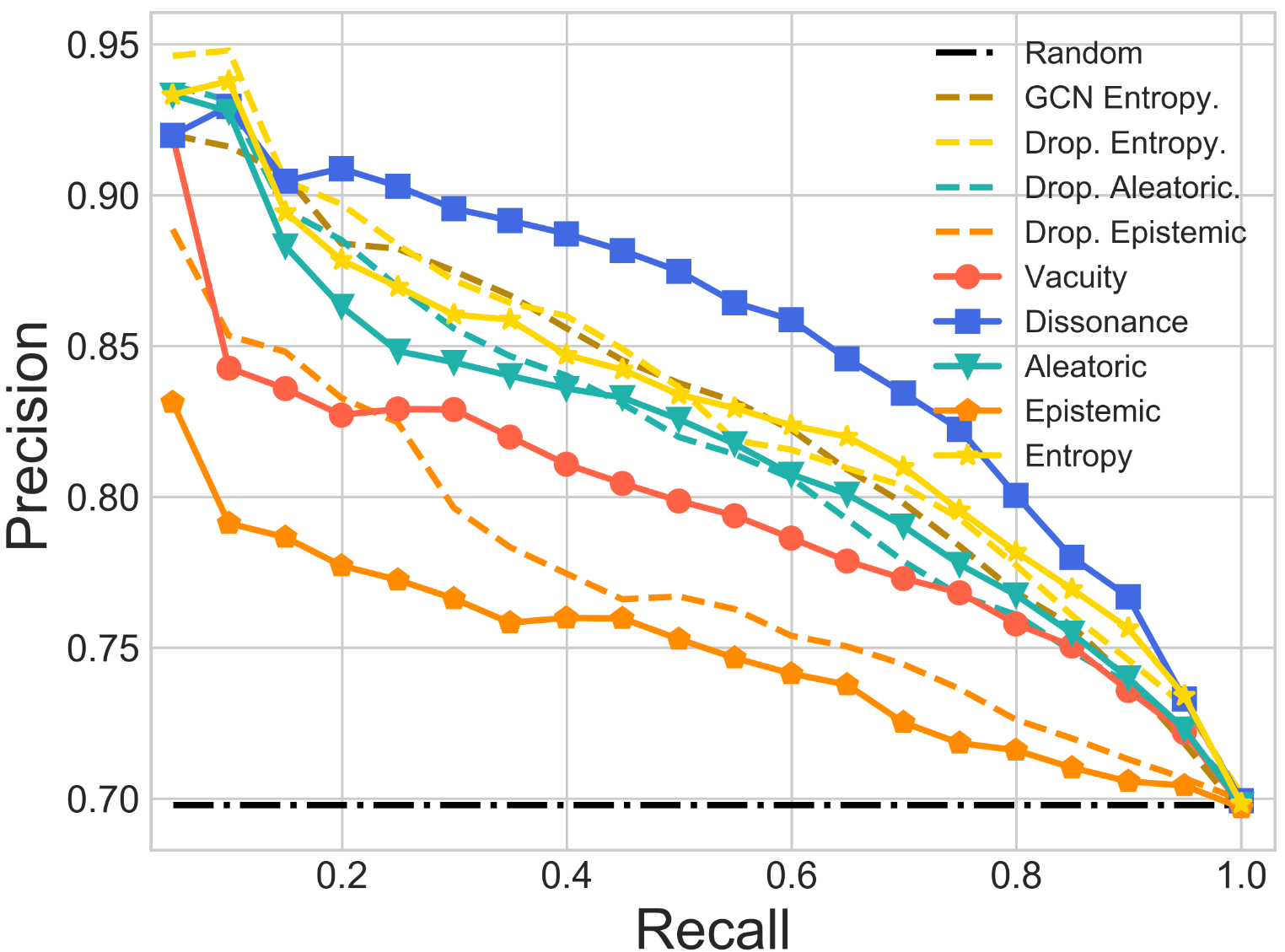}
        \caption{PR curves on Citeseer}
    \end{subfigure}
    \hfill
    \begin{subfigure}[b]{0.32\textwidth}
        \centering
        \includegraphics[width=\linewidth]{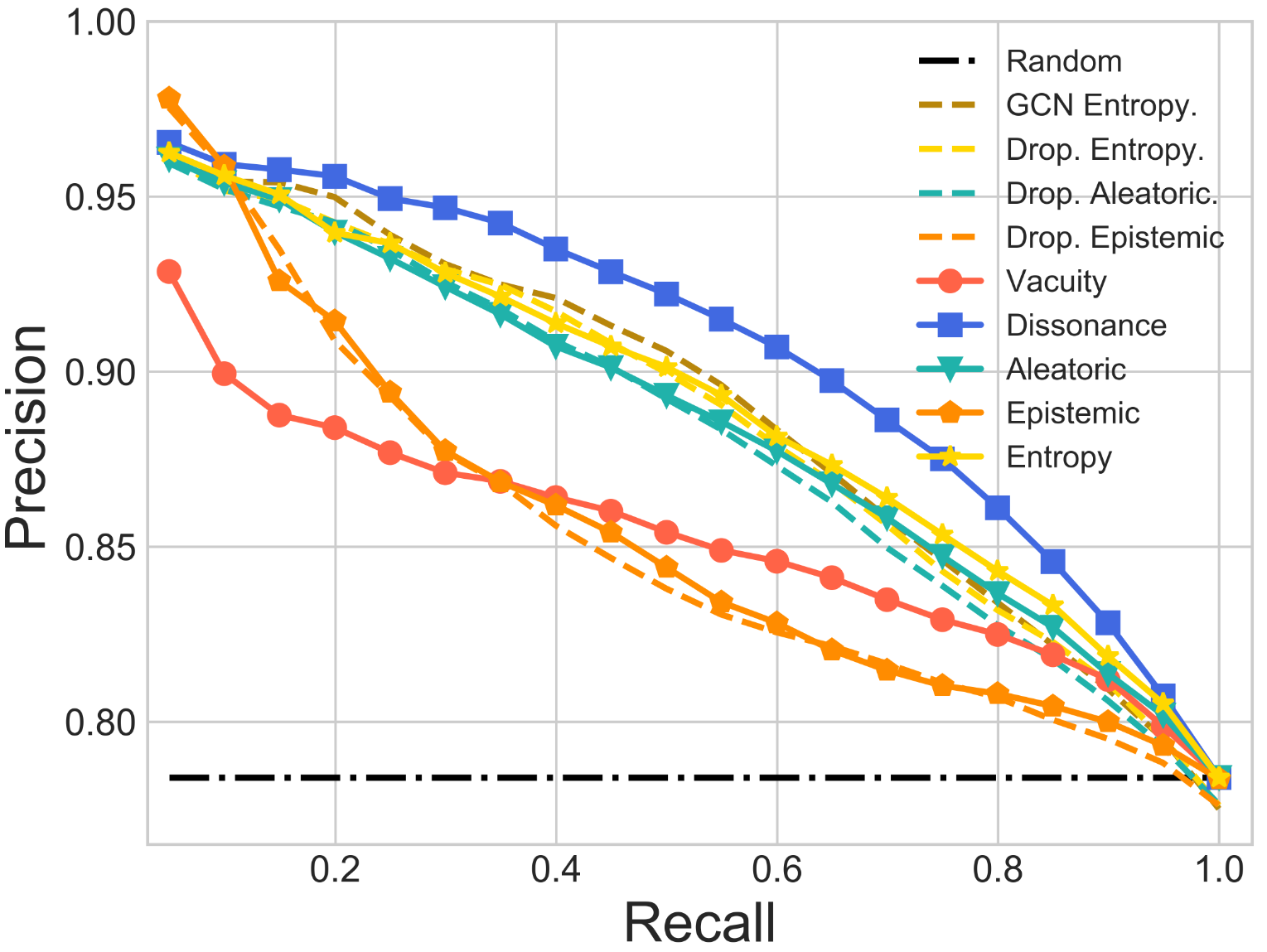}
        \caption{PR curves on Pubmed}
    \end{subfigure}
    \small{
    \caption{PR curves of misclassification detection for S-BGCN-T-K and other baselines, GCN-Drop and GCN.}
    \label{fig:pr_all}
    }
\end{figure*}

\begin{figure*}[t!]
    \centering
    \begin{subfigure}[b]{0.32\textwidth}
        \centering
        \includegraphics[width=\linewidth]{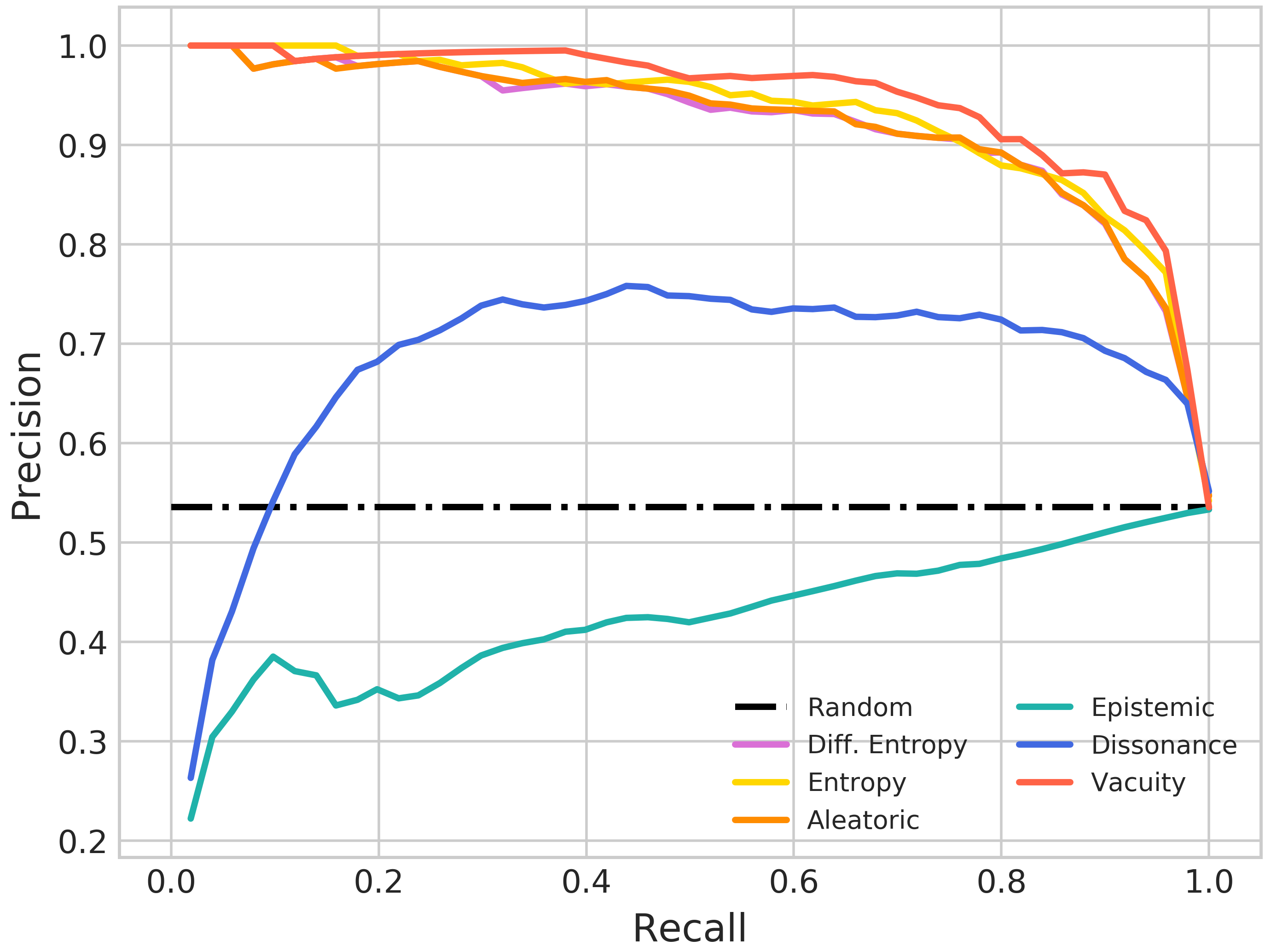}
        \caption{ Amazon Photo}
    \end{subfigure}
    \hfill
    \begin{subfigure}[b]{0.32\textwidth}
        \centering
        \includegraphics[width=\linewidth]{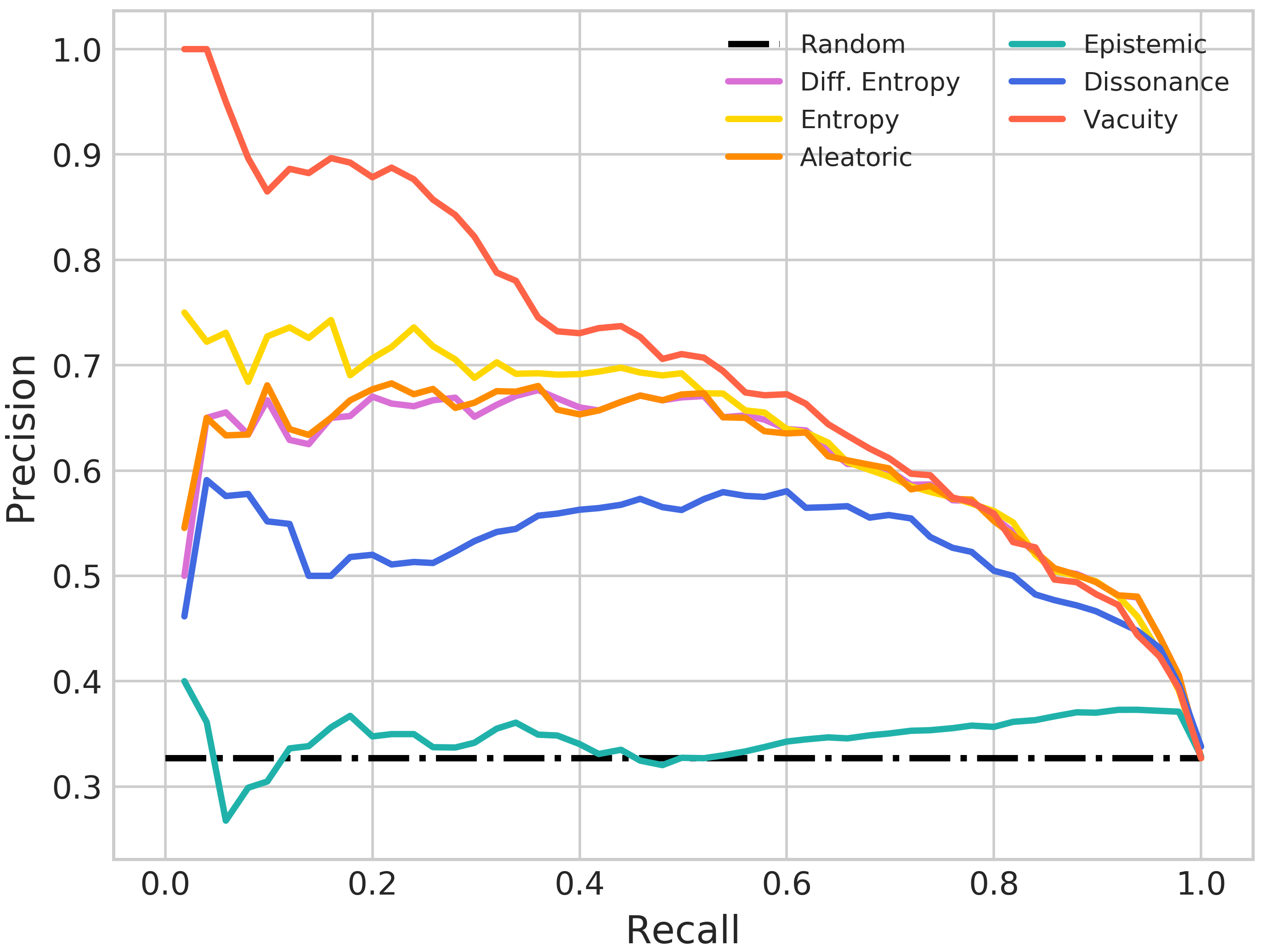}
        \caption{Amazon Computers}
    \end{subfigure}
    \hfill
    \begin{subfigure}[b]{0.32\textwidth}
        \centering
        \includegraphics[width=\linewidth]{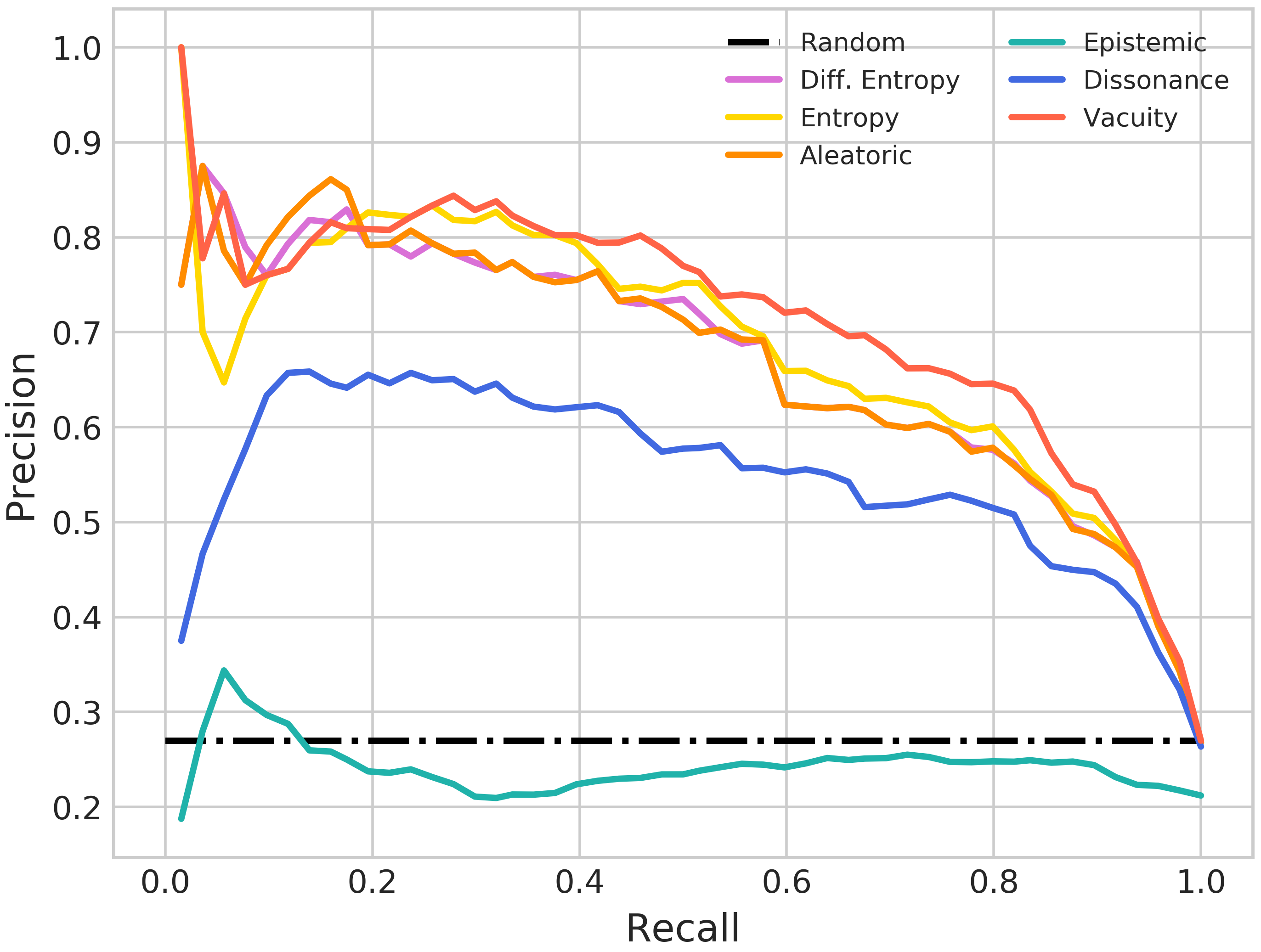}
        \caption{Coauthor Physics}
    \end{subfigure}
    \small{
    \caption{PR cuves of OOD detection for S-BGCN-T-K with uncertainties.}
    \label{fig:aupr_ood}
    }
\end{figure*}

\begin{figure*}[t!]
    \centering
    \begin{subfigure}[b]{0.32\textwidth}
        \centering
        \includegraphics[width=\linewidth]{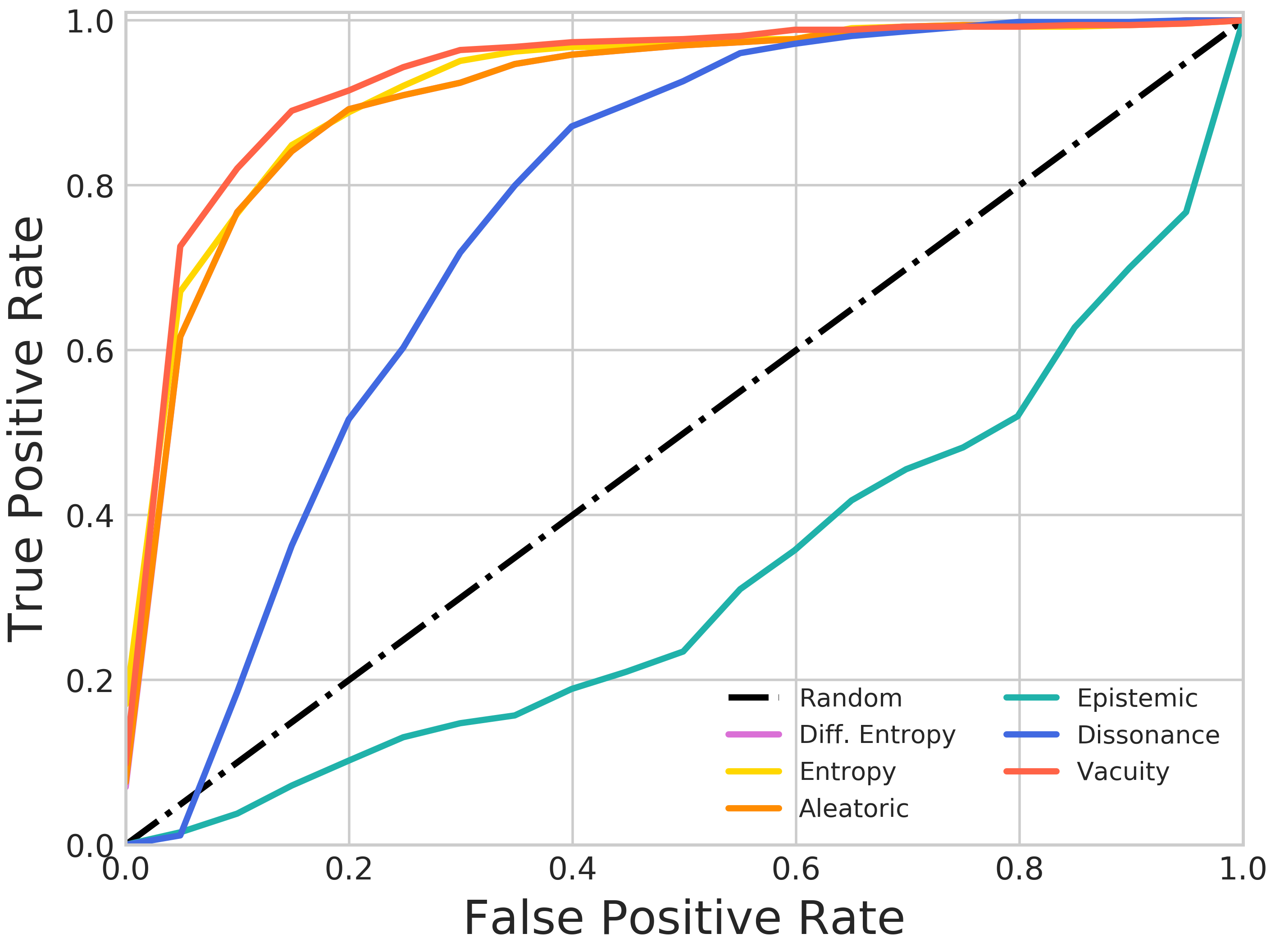}
        \caption{Amazon Photo}
        \label{fig:auroc_ood:a}
    \end{subfigure}
    \hfill
    \begin{subfigure}[b]{0.32\textwidth}
        \centering
        \includegraphics[width=\linewidth]{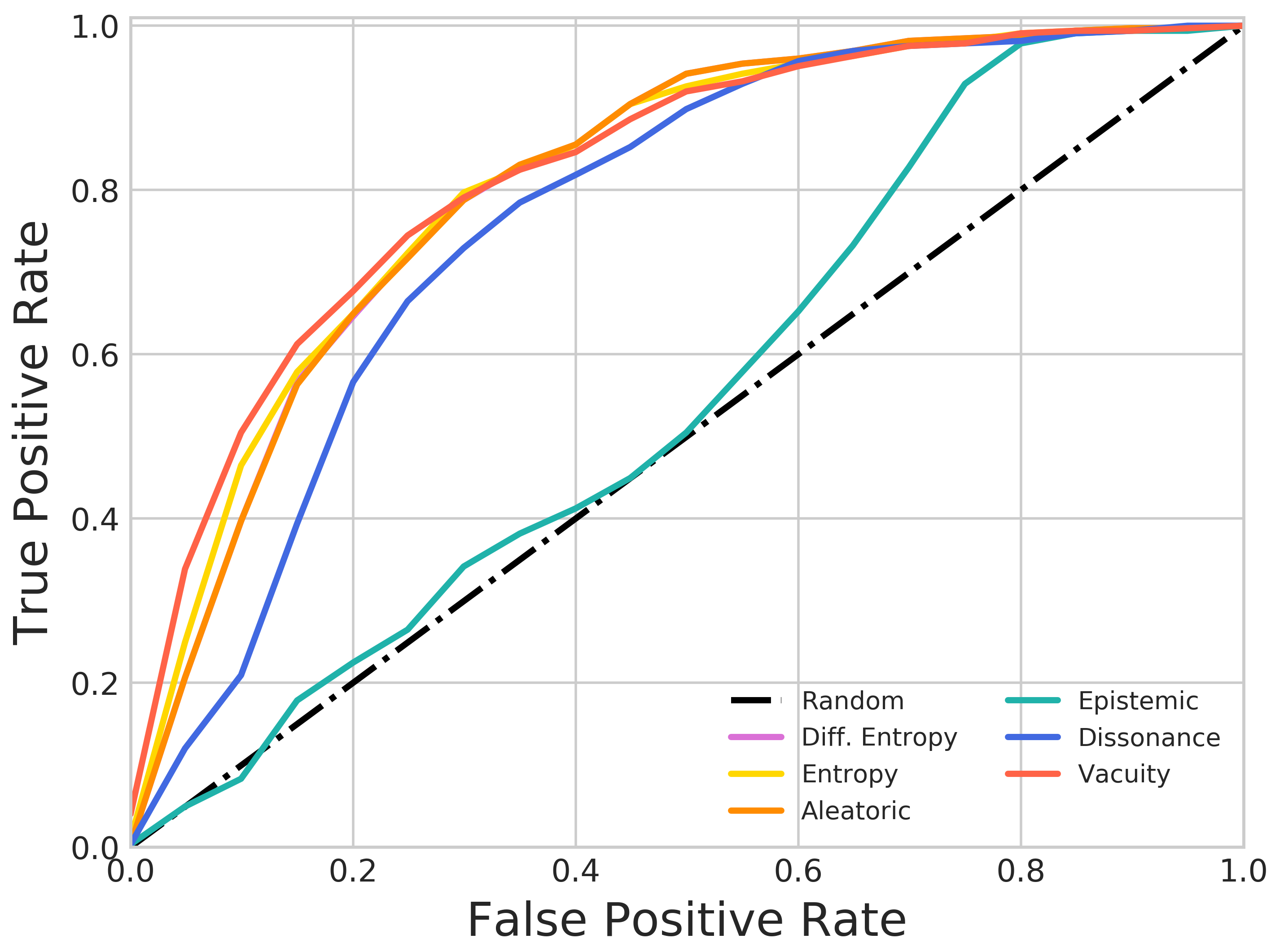}
        \caption{Amazon Computers}
        \label{fig:auroc_ood:b}
    \end{subfigure}
    \hfill
    \begin{subfigure}[b]{0.32\textwidth}
        \centering
        \includegraphics[width=\linewidth]{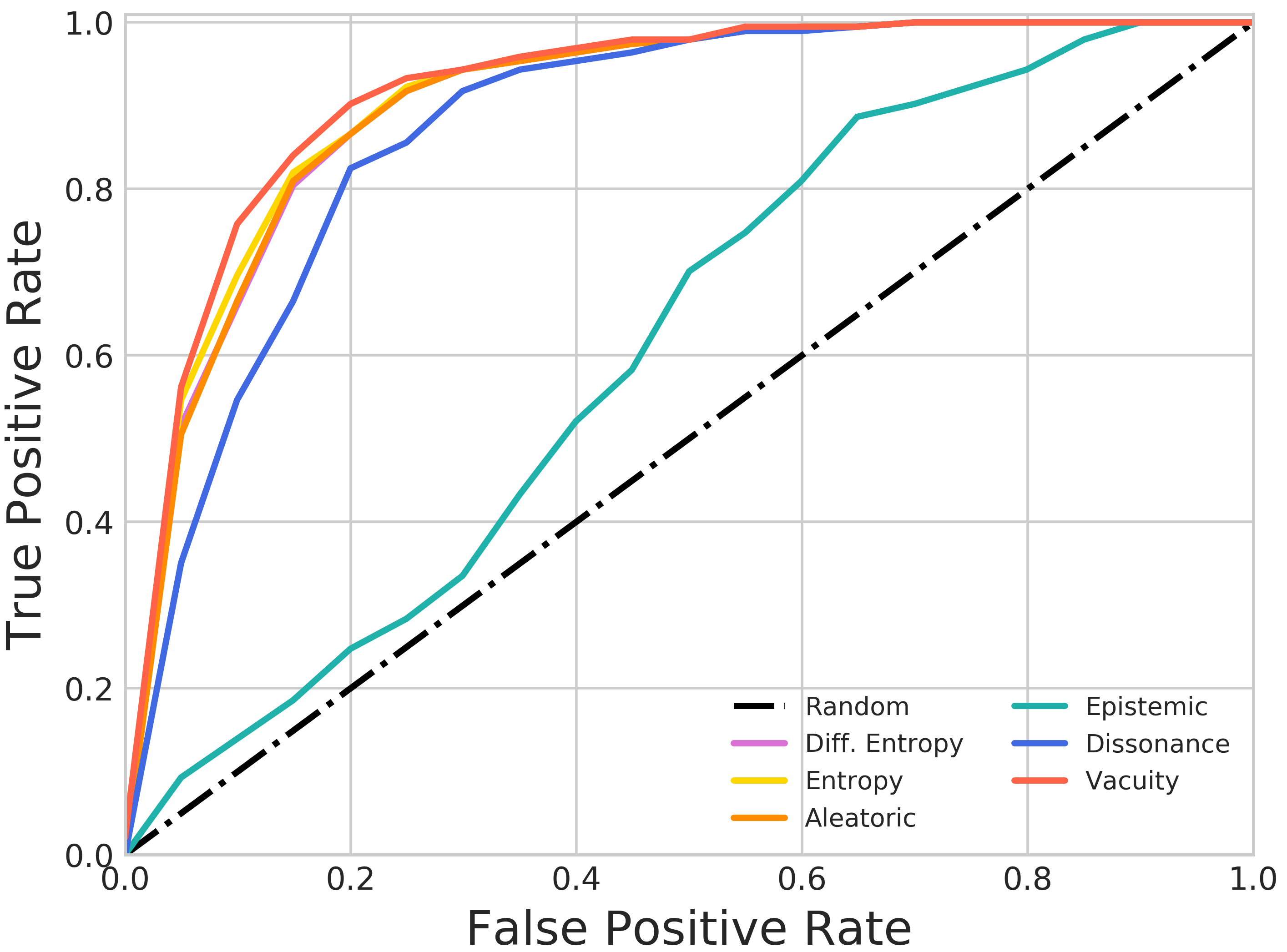}
        \caption{Coauthor Physics}
        \label{fig:auroc_ood:c}
    \end{subfigure}
    \small{
    \caption{ROC curves of OOD detection for S-BGCN-T-K with uncertainties.}
    \label{fig:auroc_ood}
}
\end{figure*}

\subsection{PR and ROC Curves}
\noindent {\bf AUPR for the OOD Detection}: Figure~\ref{fig:aupr_ood} shows the AUPRC for the OOD detection when S-BGCN-T-K is used to detect OOD in which test nodes are considered based on their high uncertainty level, given a different uncertainty type, such as vacuity, dissonance, aleatoric, epistemic, or entropy (or total uncertainty). Also to check the performance of the proposed models with a baseline model, we added S-BGCN-T-K with test nodes randomly selected (i.e., Random).

Obviously, in Random baseline, precision was not sensitive to increasing recall while in S-BGCN-T-K (with test nodes being selected based on high uncertainty) precision decreases as recall increases. But although most S-BGCN-T-K models with various uncertainty types used to select test nodes shows sensitive precision to increasing recall (i.e., proving uncertainty being an indicator of improving OOD detection). In addition, unlike AUPR in misclassification detection, which showed the best performance in S-BGCN-T-K Dissonance (see Figure~\ref{fig:pr_all}), S-BGCN-T-K Dissonance showed the second worst performance among the proposed S-BGCN-T-K models with other uncertainty types. This means that less conflicting information does not help OOD detection. On the other hand, overall we observe{\feng d} Vacuity performs the best among all. From this finding, we can claim that to improve OOD detection, less information with \feng{a} high vacuity value can help boost the accuracy of the OOD detection. 

\noindent {\bf AUROC for the OOD Detection}: First, we investigated the performance of our proposed S-BGCN-T-K models when test nodes are selected based on seven different criteria (i.e., uncertainty measures). For AUROC in Figure~\ref{fig:auroc_ood}, we observed much better performance in most S-BGCN-T-K models with all uncertainty types except epistemic uncertainty. 

\subsection{Analysis for Epistemic Uncertainty in OOD Detection}
\vspace{-1mm}
Although epistemic uncertainty is known to be effective to improve OOD detection~\cite{gal2016dropout, kendall2017uncertainties} in computer vision applications, our results demonstrate it is less effective than our vacuity-based approach. One potential reason is that the previous success of epistemic in computer vision applications are only applied in supervised learning, but they are not sufficiently validated in semi-supervised learning. 

To back up our conclusion, designe a image classification experiment 
based on MC-Drop\cite{gal2016dropout} method to do the following experiment: 1) supervised learning on MNIST dataset with 50 labeled images; 2) semi-supervised learning (SSL) on MNIST dataset with 50 labeled images and 49950 unlabeled images, while there are 50\% OOD images (24975 FashionMNIST images) in unlabeled set. For both experiment, we test the epistemic uncertainty on 49950 unlabeled set (50\% In-distribution (ID) images and 50\% OOD images). We conduct the experiment the experiment based on three popular SSL methods, VAT~\cite{miyato2018virtual}, Mean Teacher~\cite{tarvainen2017mean} and pseudo label~\cite{lee2013pseudo}.
\begin{table*}[ht]
\caption{Epistemic uncertainty for semi-supervised image classification.}
\centering
\small
  \begin{tabular}{lcccc}
    \toprule
    Epistemic & \textbf{Supervised } & \textbf{VAT}  & \textbf{Mean Teacher} & \textbf{Pseudo Label}\\
    \midrule
    \textbf{In-Distribution} & 0.140 &  \textbf{0.116}& \textbf{0.105} & \textbf{0.041}  \\
    \textbf{Out-of-Distribution}  & \textbf{0.249} & 0.049 & 0.076 & 0.020 \\
    \bottomrule
  \end{tabular}
\label{tab:ssl}
\end{table*}
Table~\ref{tab:ssl} shows the average epistemic uncertainty value for in-distribution samples and OOD samples. The result shows the same pattern with~\cite{kendall2017uncertainties, kendall2015bayesian} in a supervised setting, but an opposite pattern in a semi-supervised setting that low epistemic of OOD samples, which is less effective top detect OOD.
Note that the SSL setting is similar to our semi-supervised node classification setting, which feed the unlabeled sample to train the model.

\subsection{Compare with Bayesian GCN baseline}
Compare with a (Bayesian) GCN baseline, Dropout+DropEdge~\cite{rong2019dropedge}. As shown in the table~\ref{AUPR:dropedge} below, our proposed method performed better than Dropout+DropEdge on the Cora and Citeer datasets for misclassificaiton detection. A similar trend was observed for OOD detection.

\begin{table*}[ht!]
\vspace{-0.5mm}   
\small
\caption{Compare with DropEdge on Misclassification Detection .}
\vspace{-2mm}
\centering
\begin{tabular}{c||c|ccccc|ccccc}
\hline
\multirow{2}{*}{Dataset} & \multirow{2}{*}{Model} & \multicolumn{5}{c|}{AUROC} & \multicolumn{5}{c}{AUPR}  \\  
                  &                   & Va.\tnote{*}& Dis. & Al. & Ep. & En. & Va. & Dis. &  Al. & Ep. &En. \\ \hline
\multirow{2}{*}{Cora} & S-BGCN-T-K & 70.6 & \textbf{82.4} & 75.3 & 68.8 & 77.7&  90.3  & \textbf{95.4} & 92.4 & 87.8 &93.4 \\ 
                  & DropEdge &  -  &  -  & 76.6   &  56.1  & 76.6  & -  &  -  &  93.2  & 85.4 &93.2 \\ 
                  \hline
\multirow{2}{*}{Citeseer} &  S-BGCN-T-K & 65.4& \textbf{74.0}& 67.2& 60.7&  70.0& 79.8 & \textbf{85.6} & 82.2 & 75.2 &83.5  \\ 
                  &                  DropEdge &   - &  -  & 71.1 & 51.2 &   71.1 &   -  & -& 84.0 & 70.3  &84.0   \\ 
                  \hline
\end{tabular}
\begin{tablenotes}\scriptsize
\centering
\item[*] Va.: Vacuity, Dis.: Dissonance, Al.: Aleatoric, Ep.: Epistemic, En.: Entropy 
\end{tablenotes}
\vspace{2mm}
\label{AUPR:dropedge}
\vspace{-5mm}
\end{table*}

\section{Derivations for Joint Probability and KL Divergence}

\subsection{Joint Probability}
At the test stage, we infer the joint probability by:
\vspace{-3mm}
\begin{eqnarray}
p(\y |A, \rr; \mathcal{G}) &=& \int \int \text{Prob}(\y | \p) \text{Prob}(\p | A, \rr; \bm{\theta} ) \text{Prob}(\bm{\theta} | \mathcal{G}) d \p d\bm{\theta} \nonumber\\
&\approx& \int \int \text{Prob}(\y | \p) \text{Prob}(\p | A, \rr; \bm{\theta} ) q(\bm{\theta}) d \p d\theta  \nonumber\\
&\approx& \frac{1}{M}\sum_{m=1}^M \int \text{Prob}(\y | \p) \text{Prob}(\p | A, \rr;\bm{\theta}^{(m)} ) d \p, \quad  \bm{\theta}^{(m)} \sim q( \bm{\theta}) \nonumber \\
&\approx& \frac{1}{M}\sum_{m=1}^M \int \sum_{i=1}^N \text{Prob}(\y_i | \p_i) \text{Prob}(\p_i | A, \rr;\bm{\theta}^{(m)} ) d \p_i, \quad  \bm{\theta}^{(m)} \sim q( \bm{\theta}) \nonumber \\
&\approx& \frac{1}{M}\sum_{m=1}^M  \sum_{i=1}^N \int \text{Prob}(\y_i | \p_i) \text{Prob}(\p_i | A, \rr;\bm{\theta}^{(m)} ) d \p_i, \quad \bm{\theta}^{(m)} \sim q( \bm{\theta}) \nonumber \\
&\approx& \frac{1}{M}\sum_{m=1}^M  \prod_{i=1}^N \int \text{Prob}(\y_i | \p_i)  \text{Dir}(\p_i|\bm{\alpha}_i^{(m)}) d \p_i,\quad  \bm{\alpha}^{(m)} = f(A, \rr, \bm{\theta}^{(m)}),q\quad \bm{\theta}^{(m)} \sim q( \bm{\theta}), \nonumber 
\end{eqnarray}
where the posterior over class label $p$ will be given by the mean of the Dirichlet:
\begin{eqnarray}
\text{Prob}(y_i = p | \bm{\theta}^{(m)}) = \int \text{Prob}(y_i =p | \p_i) \text{Prob}(\p_i | A, \rr;\bm{\theta}^{(m)} ) d \p_i = \frac{\alpha_{ip}^{(m)}}{\sum_{k=1}^K \alpha_{ik}^{(m)}}. \nonumber
\end{eqnarray}
The probabilistic form for a specific node $i$ by using marginal probability,
\begin{eqnarray}
\text{Prob}(\y_i | A, \rr; \mathcal{G}) &=& \sum_{y\setminus y_i} \text{Prob}(\y | A, \rr; \mathcal{G})  \nonumber \\
&=& \sum_{y\setminus y_i} \int \int \prod_{j=1}^N\text{Prob}(\y_j | \p_j) \text{Prob}(\p_j | A, \rr; \bm{\theta} ) \text{Prob}(\bm{\theta} | \mathcal{G}) d \p d\bm{\theta} \nonumber \\
&\approx& \sum_{y\setminus y_i} \int \int \prod_{j=1}^N\text{Prob}(\y_j | \p_j) \text{Prob}(\p_j | A, \rr; \bm{\theta} ) q(\bm{\theta})d \p d\bm{\theta} \nonumber \\
&\approx&\sum_{m=1}^M \sum_{y\setminus y_i} \int \prod_{j=1}^N\text{Prob}(\y_j | \p_j) \text{Prob}(\p_j | A, \rr; \bm{\theta}^{(m)} )  d \p,\quad \bm{\theta}^{(m)} \sim q( \bm{\theta})  \nonumber \\
&\approx& \sum_{m=1}^M \Big[\sum_{y\setminus y_i} \int \prod_{j=1}^N\text{Prob}(\y_j | \p_j) \text{Prob}(\p_j | A, \rr; \bm{\theta}^{(m)} )  d \p_j\Big], \quad \bm{\theta}^{(m)} \sim q( \bm{\theta})  \nonumber \\
&\approx& \sum_{m=1}^M  \Big[\sum_{y\setminus y_i}  \prod_{j=1, j\neq i}^N\text{Prob}(\y_j | A, \rr_j; \bm{\theta}^{(m)})   \Big] \text{Prob}(\y_i | A, \rr; \bm{\theta}^{(m)} ),\quad \bm{\theta}^{(m)} \sim q( \bm{\theta})  \nonumber \\
&\approx& \sum_{m=1}^M \int \text{Prob}(\y_i | \p_i) \text{Prob}(\p_i |A, \rr; \bm{\theta}^{(m)} ) d \p_i , \quad \bm{\theta}^{(m)} \sim q( \bm{\theta}).  \nonumber
\end{eqnarray}
To be specific, the probability of label $p$ is,
\begin{eqnarray}
\text{Prob}(y_i=p | A, \rr; \mathcal{G}) \approx \frac{1}{M} \sum_{m=1}^M \frac{\alpha_{ip}^{(m)}}{\sum_{k=1}^K \alpha_{ik}^{(m)}}, \quad  \bm{\alpha}^{(m)} = f(A, \rr, \bm{\theta}^{(m)}),\quad \bm{\theta}^{(m)} \sim q( \bm{\theta}). \nonumber
\end{eqnarray}
\subsection{KL-Divergence }
KL-divergence between $\text{Prob}({\bf y} | {\bf r}; \bm{\gamma},  \mathcal{G})$ and  $\text{Prob}({\bf y} | \hat{\p})$ is given by 
\begin{eqnarray}
\text{KL}[\text{Prob}(\y | A, \rr;\mathcal{G})|| \text{Prob}({\bf y} | \hat{\p}))]
&=& \mathbb{E}_{\text{Prob}(\y | A, \rr;\mathcal{G})}\Big[\log \frac{\text{Prob}(\y | A, \rr;\mathcal{G})} {\text{Prob}({\bf y} | \hat{\p})} \Big] \nonumber \\
&\approx&\mathbb{E}_{\text{Prob}(\y | A, \rr;\mathcal{G})} \Big[\log \frac{\prod_{i=1}^N  \text{Prob}(\y_i | A, \rr; \mathcal{G})}{\prod_{i=1}^N  \text{Prob}({\bf y} | \hat{\p})} \Big] \nonumber \\
&\approx&  \mathbb{E}_{\text{Prob}(\y | A, \rr;\mathcal{G})} \Big[\sum_{i=1}^N \log\frac{\text{Prob}(\y_i | A, \rr; \mathcal{G})}{\text{Prob}({\bf y} | \hat{\p})} \Big] \nonumber\\
&\approx& \sum_{i=1}^N \mathbb{E}_{\text{Prob}(\y | A, \rr;\mathcal{G})} \Big[ \log\frac{\text{Prob}(\y_i | A, \rr; \mathcal{G})}{\text{Prob}({\bf y} | \hat{\p})} \Big] \nonumber\\
&\approx& \sum_{i=1}^N \sum_{j=1}^K \text{Prob}(y_i=j | A, \rr; \mathcal{G}) \Big( \log\frac{\text{Prob}(y_i=j | A, \rr; \mathcal{G})}{\text{Prob}(y_i=j| \hat{\p})} \Big) \nonumber
\end{eqnarray}

The KL divergence between two Dirichlet distributions $\text{Dir}(\alpha)$ and $\text{Dir}(\hat{\alpha})$ can be obtained in closed form as,
\begin{eqnarray}
\text{KL}[\text{Dir}(\alpha)\| \text{Dir}(\hat{\alpha})] = \ln \Gamma(S) - \ln \Gamma(\hat{S}) + \sum_{c=1}^K \big(\ln \Gamma(\hat{\alpha}_c) - \ln \Gamma(\alpha_c) \big)   + \sum_{c=1}^K (\alpha_c - \hat{\alpha}_c)(\psi(\alpha_c) - \psi(S)),   \nonumber
\end{eqnarray}
where $S = \sum_{c=1}^K \alpha_c$ and $\hat{S} = \sum_{c=1}^K \hat{\alpha}_c$.

\end{document}